%% file: main.tex
\pdfoutput=1
\documentclass[11pt]{article}

\usepackage[margin=1in]{geometry}
\usepackage[T1]{fontenc}
\usepackage[utf8]{inputenc}
\usepackage{lmodern}
\usepackage{microtype}
\usepackage{amsmath,amssymb,amsthm,mathtools}
\usepackage{booktabs}
\usepackage{array}
\usepackage{xcolor}
\usepackage{caption}
\usepackage{tikz}
\usepackage{pgfplots}
\usepackage{flafter}
\usepackage{placeins}
\usepackage{enumitem}
\usepackage[numbers,sort&compress]{natbib}
\usepackage{hyperref}
\usepackage[nameinlink,capitalise]{cleveref}

\usetikzlibrary{arrows.meta,backgrounds,calc,decorations.pathreplacing,fit,
  patterns,positioning,shapes.geometric}
\input{figures/figure-style}

\hypersetup{
  colorlinks=true,
  linkcolor=blue!55!black,
  citecolor=green!35!black,
  urlcolor=blue!65!black,
  pdftitle={Toward Machine Learning with the Unit as a Primitive: Learning from Unit-Linked Events},
  pdfauthor={Heyang Gong}
}

\theoremstyle{definition}
\newtheorem{definition}{Definition}
\newtheorem{assumption}{Assumption}
\newtheorem{proposition}{Proposition}
\newtheorem{lemma}[proposition]{Lemma}
\newtheorem{corollary}[proposition]{Corollary}
\crefname{assumption}{assumption}{assumptions}
\Crefname{assumption}{Assumption}{Assumptions}

\newcommand{\cU}{\mathcal{U}}
\newcommand{\cZ}{\mathcal{Z}}

\title{Toward Machine Learning with the Unit as a Primitive:\\
Learning from Unit-Linked Events}
\author{Heyang Gong}
\date{August 26, 2026}

\begin{document}
\maketitle

\begin{abstract}
Machine learning is usually formalized through samples, while the persistent
individual to which multiple observed or possible events refer often remains
implicit. We propose the \emph{unit} as an explicit primitive at the level of
task semantics. A learning task first declares a population of persistent
referents and a sameness criterion; the realized value $u$ denotes the selected
referent. Supervised learning is the main formal specialization. Its semantic
object is a family of unit-conditioned response laws. Homogeneity is the special
case in which those laws coincide; a sample-only conditional is silent as to
whether the world is homogeneous or the observed law is only the marginal of a
heterogeneous family. What is learned from data is a pair $(T_\phi,R_\theta)$: a
tokenizer that produces a contextual unit token and one shared response-law form
that reads it. The structured class takes that form to be a simple relation in
the token; a linear predictor is the running instance.
The token is the learner-side representation through which the task-side unit
affects prediction, while a learner specification that omits unit information is
unit-insensitive; homogeneity remains a property of
the world-side response family. When identity is unresolved, the world-side law
mixes unit-conditioned targets, while the learner composes its shared form with a
token. A trusted resolver may fix the unit and supply a lookup token; otherwise
\emph{unit abduction} forms a token of the same type from factual
evidence. Unlinked single-row observations can fail to distinguish a
heterogeneous unit world from a homogeneous pooled world; trusted same-unit
pairs separate a restricted witness. The formal results concern this supervised
specialization.
\end{abstract}

\section{Introduction}

The standard supervised-learning formulation starts from records
$\mathcal D=\{(x_i,y_i)\}_{i=1}^{N}$ and a map
$f:\mathcal X\to\mathcal Y$
\citep{bishop2006pattern,hastie2009elements,shalevshwartz2014understanding}.
The index $i$ identifies a record, not the
individual to which that record refers. If two records concern the same
patient, user, device, organism, or other persistent individual, ordinary
notation represents this relation only through additional metadata or an
implicit modeling convention. This omission is harmless for some row-level
questions, but consequential when events share an individual, queries change
while that individual is held fixed, or event-to-individual attribution is
uncertain. It also leaves a more basic question unstated: whether different
individuals share one response law or whether the observed row-level law pools
distinct unit-conditioned relationships.

Persistent individuals are not new to statistics or machine learning.
Repeated-measures and random-effects models link observations through a
supplied subject or group and model subject-specific variation
\citep{laird1982random,gelman2006data}. Potential-outcome frameworks compare
alternative treatment responses for the same experimental unit
\citep{rubin1974causal}. Recommender systems use supplied user IDs to address
user-specific representations learned across interactions
\citep{mnih2007pmf,koren2009matrix}, while record linkage and entity resolution
model uncertainty about whether records concern the same underlying entity
\citep{steorts2016entity}. Each tradition develops a rich local theory with its
own assumptions and targets.

Together, these traditions suggest a common machine-learning formulation with a
semantic order that is usually left implicit. A task first declares a population
of persistent referents and a sameness criterion. It then declares the learning
object attached to those referents and the across-unit structure that supports
joint learning. Only after these choices does learner access become relevant:
attribution may be supplied directly or inferred from evidence. We call the
persistent referent a \emph{unit} and elevate it to a primitive of the learning
problem. Let
\begin{equation}
  U:\Omega\longrightarrow\cU,
  \qquad
  U\sim\Pi,
  \qquad
  U=u,\quad u\in\cU,
  \label{eq:unit-population}
\end{equation}
represent population-to-individual selection and its realization.

\paragraph{Unit declaration.}
A learning task declares a population of possible units, a criterion under
which observed or possible events concern the same persistent referent, and the
span over which that relation is retained. The realization $U=u$ denotes that
referent. The law $\Pi$ describes population selection; the data-collection
protocol separately determines the joint law of dataset attributions. The unit
declaration precedes the task-specific choice of learning object and does not by
itself impose a response law, loss, model class, or causal semantics.

The organizing claim of the paper is that \emph{machine learning learns shared
structure across task-declared units from noisy, selectively observed,
unit-linked events}. The primitive states what persists and which events belong
together. Supervised learning is the main formal specialization: a fixed unit
selects an entire unit-conditioned response law,
\begin{equation}
  \begin{aligned}
  u &\longmapsto \bigl[\,x\mapsto P_u^\star(dy\mid x)\,\bigr],\\[-1pt]
  P_u^\star(dy\mid x)
    &:=P^\star(Y\in dy\mid X=x,U=u).
  \end{aligned}
  \label{eq:unit-selects-law}
\end{equation}
Changing $x$ queries another location on the same response surface; changing
$u$ selects another member of the family. A family of unrelated maps
$\{R_u\}$ is too flexible to learn from finite data, including for a previously
unseen unit. Joint learning therefore requires a shared restriction across
units, developed in \cref{sec:response-law-constraints}: a tokenizer produces a
contextual unit token, and one shared response-law form reads it.

\begin{figure}[t]
  \centering
  \input{figures/fig1-events-unit}
  \caption{\textbf{From the unit declaration to linked events and response laws.}
  The task specifies the unit boundary and sameness criterion; $U\sim\Pi$ is the
  population statement and $U=u$ its realized unit. Several event records may
  be attributed to the same $u$, and that same referent indexes a world-side
  response law. Holding $u$ fixed preserves the referent while allowing context
  and event variation.}
  \label{fig:events-unit}
\end{figure}
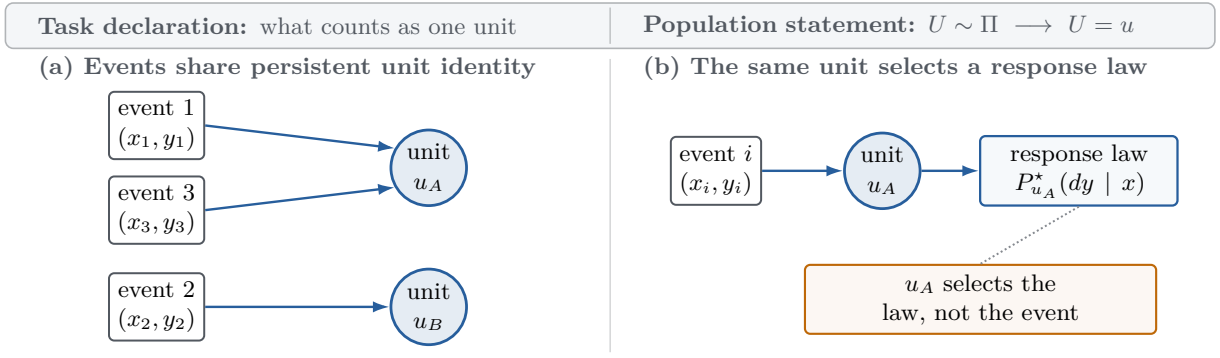

The paper makes three contributions.
\begin{enumerate}[leftmargin=*,itemsep=3pt]
  \item It formalizes the unit as a task-declared population primitive: a
  sample records an event, and the task first declares the persistent referent
  to which events belong.
  \item It states a shared-form interface
  (\cref{hyp:shared-form-token}) whose structured class is a simple relation in
  the token, with a linear predictor as the running instance
  (\cref{def:linear-shared-form}); the learned object is the pair
  $(T_\phi,R_\theta)$.
  \item It separates oracle predictive value, evidence access, learner
  approximation, and a single-row impossibility boundary, with evaluation
  corollaries for row versus unit weighting and for record-wise versus
  unit-disjoint splits.
\end{enumerate}

\section{The Unit as a Machine-Learning Primitive}
\label{sec:unit-primitive}

\subsection{Samples, Events, and Task-Declared Units}

For the supervised specialization developed below, a conceptual complete-data
representation with ground-truth attribution is
\begin{equation}
  \mathcal D_U^{\rm world}
  =\{(x_i,y_i;u_i)\}_{i=1}^{N},
  \qquad u_i\in\cU.
  \label{eq:world-unit-table}
\end{equation}
The input--response pair $(x_i,y_i)$ is the event content used in this
specialization. The semicolon marks the more general role distinction between
that event content and $u_i$, which identifies the individual to which the event
is attributed.
The sample index remains only bookkeeping. In particular,
\begin{equation}
  u_i=u_j
  \quad\Longleftrightarrow\quad
  \text{events $i$ and $j$ concern the same individual},
  \label{eq:same-unit}
\end{equation}
whenever the protocol supplies reliable attribution. Equality records the same
persistent referent. Predictive equivalence is weaker: two distinct units may
induce the same response law.

The space $\cU$ is task-declared and need not be numerical. Direct access
exposes three objects:
\begin{equation}
  \underbrace{k}_{\text{trusted key}},
  \qquad
  \underbrace{u(k)}_{\text{task-side unit}},
  \qquad
  \underbrace{Z_{u(k)}^c}_{\text{learner-side lookup token}}.
  \label{eq:direct-address-content}
\end{equation}
The map \(k\mapsto u(k)\) resolves the referent, and lookup supplies the token
through which $R_\theta$ reads it in context $c$. In an ID-indexed model
\begin{equation}
  Z_{u(k)}^c
  \;=\;
  z_\theta(k),
  \label{eq:lookup-token-row}
\end{equation}
an embedding, preference factor, random effect, or other parameter block, with
any remaining answer-time variation absorbed into \(c\). Its coordinates are
model-internal and need not be unique; the semantic and inferential distinctions
are summarized in \cref{app:interface-distinctions}
\citep{bengio2013representation}.

What is carried consistently from attribution into task-specific learning is
the resolved referent. All records attributed to \(u(k)\) read or update the
token associated with that same key. Accordingly, \(P_\theta(dy\mid x,u)\)
semantically denotes a learned response law indexed by a fixed unit;
computationally it is the shared form \(R_\theta\) evaluated at that unit's
token \(Z_u^c\). In an ID-indexed model it may be computed from
\(z_\theta(k)\). A dynamic history-dependent response state, when used, belongs
to the tokenizer's context argument or to the response model, and may change
while the referent stays fixed.

When neither unit attribution nor a trusted referential resolver is
learner-visible, the
observed data may contain only
\begin{equation}
  \mathcal D^{\rm obs}=\{(x_i,y_i)\}_{i=1}^{N}.
  \label{eq:observed-row-table}
\end{equation}
The complete-data representation in \cref{eq:world-unit-table} is therefore
distinct from observations with learner-visible attribution,
\begin{equation}
  \mathcal D_U^{\rm obs}=\{(x_i,y_i;u_i)\}_{i=1}^{N},
  \label{eq:observed-unit-table}
\end{equation}
  in which persistent-unit labels, or a trusted resolver that determines them,
  are available to the learner. This attribution guarantee reveals which
  referent and parameter address an event concerns; it does not make learned
  embeddings, random effects, or response states observed data.

\subsection{Fixing the unit does not fix the event}

In supervised learning, the fixed-unit response law
$P_u^\star(dy\mid x)$ from \cref{eq:unit-selects-law} may remain stochastic.
Holding $u$ fixed conditions on the same individual
while allowing repeated measurements, choices, or outcomes to differ. Event
noise, time-varying state, and other exogenous
variation can remain after unit selection. The unit primitive therefore
separates ``which individual?'' from ``what happens for that individual here?''
as shown in \cref{fig:events-unit}. Independence, causal semantics, and
deterministic responses are separate modeling commitments. The concrete
three-event inset has three samples but only two units: events~1 and~3 share
$u_A$ while their inputs and realized responses remain event-specific.

The unit declaration precedes the choice between observational and causal
semantics. Causal estimands require additional commitments; one such
specialization is recorded in \cref{app:discoscm-specialization}.
Direct unit access and unit abduction provide two learner-side modes for obtaining the
same task-declared unit (\cref{sec:unit-access}).

\section{Learning a Family of Unit-Conditioned Response Laws}
\label{sec:supervised-specialization}

This paper develops supervised learning as the main formal specialization of
the unit primitive because the ordinary object $P(dy\mid x)$ makes the
consequences of a persistent unit especially transparent. Attribution is
treated separately in \cref{sec:unit-access}. Once the unit is declared, the
semantic object is the family
$\{P_u^\star(dy\mid\cdot):u\in\cU\}$ in
\cref{eq:unit-selects-law}. The next subsection states the
learner-side factorization: every path from the unit to prediction passes
through a contextual token, all units use one response-law form $R_\theta$, and
every unit-specific difference represented by the learner must pass through the
token. The supervised question is therefore not how a
particular architecture encodes $u$, but whether a tokenizer can organize
unit-specific regularities so that one shared $R_\theta$ becomes simple enough
to generalize.
Conceptually, a fixed $u$ selects an entire response surface over the complete
task-declared response input $x$. Changing $x$ moves to another location on
that same surface while $u$ remains fixed.
Across units, the dependence may enter through a baseline shift $\alpha(u)$, a
changed input effect $\beta(u)$, a response function $f_u$, an outcome-noise law,
or the support, threshold, ranking preference, or query sensitivity of the
response. Under \cref{hyp:shared-form-token}, all of these differences are carried by the
token. Concatenation, attention, and modulation may implement the tokenizer or a
richer readout; the structured class below takes the readout $R_\theta$ to be a
simple relation in the token, with a linear predictor as the running instance.

Three layers must be distinguished. If units are declared and
$P_u^\star(dy\mid x)$ is the same for almost every $u$, the world is a
\emph{homogeneous unit extension} of ordinary supervised learning. Separately,
a heterogeneous family can be marginalized over $U$ to produce the same
row-level conditional law $P^\star(dy\mid x)$. At the computational layer, a
unit-omitting learner specification is unit-insensitive; a constant token or a
response form that ignores its token is a canonical realization. This learner-side
restriction and the world-side homogeneity layer are distinct, so row-level notation and fit alone do
not choose between them. \Cref{prop:single-row-collapse} makes the ambiguity
formal and gives one observation protocol that can separate a restricted pair of
worlds; the boundary cases are collected in
\cref{app:interface-distinctions}.

\subsection{Response Heterogeneity Across Units}

The first distinction is whether changing the unit changes that conditional
relation at all.

\begin{definition}[Unit-response homogeneity and heterogeneity]
\label{def:unit-response-heterogeneity-main}
Fix the complete response inputs declared by a supervised-learning task. The
declared family of unit-conditioned response laws is \emph{unit-response
homogeneous} if there exist a $\Pi$-null set $N\subseteq\cU$ and a common
conditional response kernel such that, for every
$u\notin N$ and every declared input $x$,
\begin{equation}
  P_u^\star(dy\mid x)
  =P^\star(dy\mid x).
  \label{eq:main-response-homogeneity}
\end{equation}
Here the right-hand side denotes the common response kernel shared across
units; it is not defined by marginalizing over $U$. For compatible conditional
versions, however, homogeneity makes this kernel equal to the unit-marginal
response law $P^\star(Y\in dy\mid X=x)$ for $P_X$-almost every $x$. Thus the two
uses of $P^\star(dy\mid x)$ agree under the homogeneous restriction.
The family is \emph{unit-response heterogeneous} otherwise.
\end{definition}

This definition concerns variation in the response law and is independent of a
particular parameterization, architecture, or clustering device. Identity and
predictive behavior remain distinct: two individuals may induce the same
response law while remaining different units. The appendix gives the jointly
measurable-kernel formulation, its deployment-visible qualification, and the
usual almost-everywhere and off-support boundaries in
\cref{def:unit-response-heterogeneity}.

The homogeneity distinction is defined only after the task has specified what
counts as a unit, and it is a world-side property independent of a particular
learner. The computational counterparts---a unit-insensitive learner, a
structured pair $(T_\phi,R_\theta)$, and a saturated private-law class---are
developed in \cref{sec:response-law-constraints}.

\subsection{Shared response-law form and contextual unit tokens}
\label{sec:response-law-constraints}

Heterogeneity characterizes world-side variation among response laws. Joint
learning requires a separate computational restriction: we do not fit an
unrelated mechanism $R_u$ for each unit. That class is too flexible.

When the response specification distinguishes a focal query $x$ from an answer-time
context $c$, write $P_u^\star(dy\mid x,c)$ for the complete conditional; $c$ may
otherwise be absorbed into a larger $x$. Given factual evidence attached to a
unit and the current context, a tokenizer produces a contextual unit token
\begin{equation}
  Z_u^c
  =
  T_\phi(\mathcal O_u,c)
  \in\cZ.
  \label{eq:unit-tokenizer}
\end{equation}

\begin{assumption}[Shared form, contextual tokens]
\label{hyp:shared-form-token}
For the learner in a given context, every path by which the selected unit can
affect prediction is mediated by a contextual unit token, and all units use one
shared response-law form. Any unit-specific difference represented by the
learner must be expressed by the token.
\end{assumption}

The unit \(u\) is the task-side referent and \(Z_u^c\) is a learner-side
representation.
\Cref{hyp:shared-form-token} is an interface. The structured class of this paper
requires a simple shared relation in the token: unit-specific differences enter
$R_\theta$ only through a declared low-complexity map of $z$, while the query
maps may remain nonlinear. A finite-dimensional linear predictor
(\cref{def:linear-shared-form}) is the running instance; other declared simple
maps of the token are the same restriction, not a second theory. Its boundary cases
are collected in \cref{app:interface-distinctions}.
The object $Z_u^c$ need not be a deterministic vector. A deterministic tokenizer
may return a vector, a distribution, or another typed computational object; a
fixed embedding \(z_\theta(k)\) is the simplest lookup special case. Alternatively,
the tokenizer itself may be stochastic, in which case it returns a law over
token values. These are different typed constructions. A distribution-valued token, say
$Z_u^c=\mu_u^c$, is passed as one typed object to
$R_\theta(dy\mid x,c,\mu_u^c)$; it is not the sampling randomness marginalized
in \cref{eq:shared-response-random-token}. In every case the token is
learner-side information about the unit under the present evidence and context,
not a second copy of the unit. It may be read loosely as a belief about the unit,
but it is not required to be a posterior over identity, and it is not required
to recover a unique complete latent unit.

The unit--token relation is a representation map, not a bijection. The same unit may
receive different tokens in different contexts,
\begin{equation}
  Z_u^c\neq Z_u^{c'}.
  \label{eq:token-context}
\end{equation}
Distinct units may receive the same token, or the same token law, in a given
context,
\begin{equation}
  u\neq v,
  \qquad
  Z_u^c\overset{d}{=}Z_v^c.
  \label{eq:token-noninjective}
\end{equation}
Token coincidence is computational indistinguishability under the current query
scope. It does not identify the units: referential identity remains $u\neq v$.

For a deterministic token, the learner uses one shared response mechanism
\begin{equation}
  \boxed{
  \widehat P_{\theta,\phi}(dy\mid\mathcal O_u;x,c)
  :=
  R_\theta(dy\mid x,c,Z_u^c).}
  \label{eq:shared-response-form}
\end{equation}
If the tokenizer is a stochastic kernel over token values, the learner instead
uses the integral
\begin{equation}
  \widehat P_{\theta,\phi}(dy\mid\mathcal O_u;x,c)
  :=
  \int
  R_\theta(dy\mid x,c,z)\,
  T_\phi(dz\mid\mathcal O_u,c).
  \label{eq:shared-response-random-token}
\end{equation}
The essential point is that $R_\theta$ is the same learned map for every unit;
neither display defines the world target $P_u^\star$.

\begin{definition}[Linear shared form]
\label{def:linear-shared-form}
Fix a finite dimension $d$ and let tokens take values in $\mathbb R^d$. A shared
form $R_\theta$ is \emph{linear in the token} when there exist maps
$\alpha_\theta(x,c)\in\mathbb R$ and $\psi_\theta(x,c)\in\mathbb R^d$ such that
\begin{equation}
  \eta_\theta(x,c,z)
  =
  \alpha_\theta(x,c)+\langle\psi_\theta(x,c),z\rangle
  \label{eq:linear-predictor}
\end{equation}
and $R_\theta(dy\mid x,c,z)$ depends on $z$ only through this linear predictor,
via a declared exponential-family, GLM, or scoring link. The maps
$\alpha_\theta$ and $\psi_\theta$ may be nonlinear in $(x,c)$. Finite-type
models arise when tokens are further restricted to a finite set or simplex in
$\mathbb R^d$; a low-rank model is the same class at small $d$.
\end{definition}

\Cref{hyp:shared-form-token} remains the interface. The structured class is a
restriction on how $R_\theta$ reads the token, not a private mechanism per unit:
$R_\theta$ depends on $z$ only through a declared simple relation, while $T_\phi$
may remain complex. \Cref{def:linear-shared-form} is the running instance.
A finite-type token or another declared low-complexity map of $z$ instantiates
the same class; they are not developed here. Misspecification
is residual unit-specific variation that the declared simple relation cannot
absorb. Without a bound on $d$, or with an unrestricted nonlinear readout of $z$,
the factorization is again a reparameterization.
Random-intercept and random-slope models, matrix factorization, and a linear
last-layer over a query encoder instantiate the linear running instance; they
differ in how $\alpha_\theta$ and $\psi_\theta$ are parameterized, not in the
role of the unit. A distribution-valued token lies outside the running instance
unless it is reduced to a vector in $\mathbb R^d$, for example a mean or a
finite coefficient summary.

When identity is unresolved,
under response sufficiency $Y\perp\!\!\!\perp\mathcal O\mid(U,X,C)$ and the
external-query condition $U\perp\!\!\!\perp(X,C)\mid\mathcal O$, the world-side target
and the learner-side token composition are
\begin{equation}
  \boxed{
  \begin{aligned}
  P^\star(dy\mid x,c,\mathcal O)
  &=
  \int_{\cU}
  P_u^\star(dy\mid x,c)\,
  P(du\mid\mathcal O),
  &&\text{world target},\\
  \widehat P_{\theta,\phi}(dy\mid\mathcal O;x,c)
  &:=
  \int_{\cZ}
  R_\theta(dy\mid x,c,z)\,
  T_\phi(dz\mid\mathcal O,c),
  &&\text{learner composition}.
  \end{aligned}}
  \label{eq:world-learner-response-coordinates}
\end{equation}
The first line is the exact world mixture under these conditions; it contains
no learned parameter. The
second line is a learner defined on token space. They coincide only if the
tokenizer realizes an appropriate token law for the relevant world-side unit
mixture and the shared response form correctly realizes the corresponding
unit-conditioned targets. A point token is a further learner-side collapse.
If a trusted key fixes $u(k)$, the world identity conditional is the Dirac law
$P(du\mid k)=\delta_{u(k)}$, while the learner separately reads the lookup token
$Z_{u(k)}^c$. The theorems below use an identity-mixture learner because it is
information-theoretically convenient; deployed learners may instead use the
token composition. These are alternative learner coordinates, not identities
that place learned parameters inside world truth.

What is learned is the pair $(T_\phi,R_\theta)$, not a collection of unrelated
unit-specific mechanisms. Heterogeneity is thereby organized in a common token
space:
\begin{equation}
  \mathcal O_u
  \xrightarrow{T_\phi}
  Z_u^c
  \xrightarrow{R_\theta(\,\cdot\,,x,c)}
  \widehat P_{\theta,\phi}(\,\cdot\mid\mathcal O_u;x,c).
  \label{eq:token-pipeline}
\end{equation}
The tokenizer is responsible for discovering structure that units can share; the
response mechanism specifies how any token, faced with a query and context,
produces a response. A good tokenizer converts complex unit-specific
regularities into a simple, stable, shared law on token space.

The shared-form class also contains a degeneration in which the response never
receives the unit. In the linear running instance this is $\psi_\theta\equiv 0$,
or any token that does not enter $\eta_\theta$. That unit-omitting specification is
the computational restriction
\begin{equation}
  \widehat P_{\theta,\phi}(dy\mid\mathcal O_u;x,c)
  =
  \widehat P_\theta(dy\mid x,c)
  \qquad\text{for all }u,
  \label{eq:unit-insensitive-learner}
\end{equation}
which has two canonical sufficient realizations under
\cref{hyp:shared-form-token}:
\begin{equation}
  Z_u^c=z_0,
  \label{eq:constant-token}
\end{equation}
or
\begin{equation}
  R_\theta(dy\mid x,c,z)
  =
  R_\theta(dy\mid x,c).
  \label{eq:token-ignored}
\end{equation}
These equations define the unit-insensitive learner specification. The same
specification can be used with a homogeneous world or with the marginal of a
heterogeneous family; world homogeneity is a property of $\{P_u^\star\}$, not
of token equality.

The tokenizer and the response learner are trained jointly: each shapes whether
the other remains simple enough to generalize.

When a trusted identity is available, a token may be obtained by direct lookup
at the resolved address. When no resolver supplies the realized unit, a token of
the same type is formed from factual evidence by unit abduction
(\cref{sec:unit-access}). Both access modes still train the same pair
$(T_\phi,R_\theta)$.

The token space also determines how observations from other units may constrain
the response law for a target unit. Referential identity answers which
individual is present; sharing observations across units requires an explicit
similarity or borrowing rule, such as a metric, kernel, graph, or hierarchy.

A learnable specification names what is shared, what varies with the unit,
what remains invariant, and how evaluation distinguishes a unit-dependent
relationship from identifier memorization
(\cref{app:evaluation-checklist}).
\Cref{hyp:shared-form-token} answers the first three by requiring a shared
$R_\theta$ and a contextual token; the structured class further requires that
$R_\theta$ depend on the token only through a declared simple relation, with
\cref{def:linear-shared-form} as the running instance.

\paragraph{Structure as the condition of learnability.}
\label{subsec:structure-learnability}
The pair $(T_\phi,R_\theta)$ is the computational expression of a dual
restriction. The response map $x\mapsto P_u^\star$ needs structure even for a
fixed unit, and the unit axis $u\mapsto P_u^\star$ needs a second restriction so
that observations can be shared and an unseen unit can be served. Homogeneity is
one such structure. A simple shared relation in the token supplies another:
units may differ, but only through a low-complexity map of $Z_u^c$. The linear
running instance uses the inner product in \cref{eq:linear-predictor}; other
declared simple maps of the same token remain open. Which other units'
observations may constrain the token is a separate borrowing question. The roles
of a which-unit belief and a trusted lookup are clarified in
\cref{app:interface-distinctions}.

\begin{figure}[!htbp]
  \centering
  \input{figures/fig-response-law-wedge}
  \caption{\textbf{The computational assumption for the supervised
  specialization.}
  (a)~A unit-insensitive learner: every unit receives a common token $z_0$,
  or the shared form ignores the token.
  (b)~Structured heterogeneity: a tokenizer forms contextual unit tokens, and
  all units share one response-law form $R_\theta$, a simple relation in the
  token (linear running instance).
  (c)~A saturated class assigns an unrelated law to each unit.}
  \label{fig:response-law-wedge}
\end{figure}
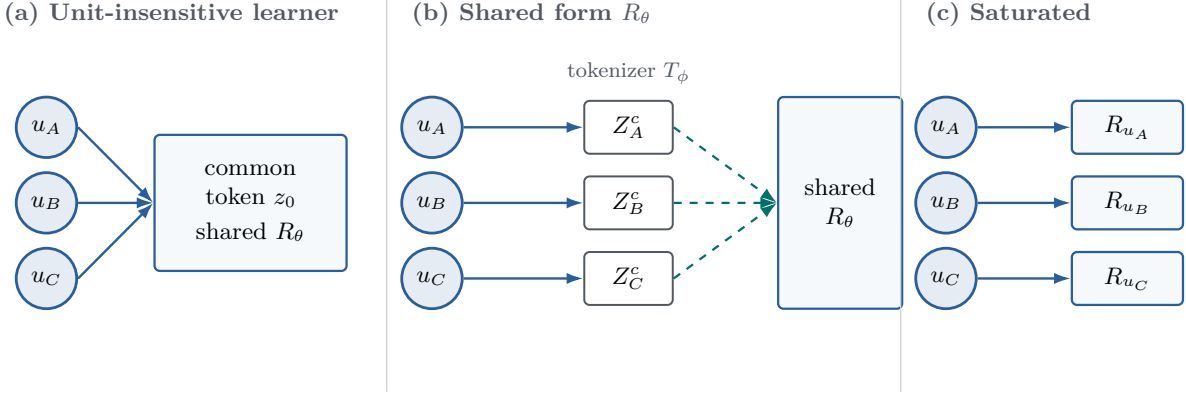
\FloatBarrier

\section{Learner Access to the Unit}
\label{sec:unit-access}

With the unit boundary, response family, and shared form $(T_\phi,R_\theta)$
fixed (\cref{subsec:structure-learnability}), world-side attribution
determines which realized unit the event concerns. Learner access determines
how that unit enters downstream learning as a contextual token. It does not
replace the structure that made the family learnable. We distinguish direct
unit access from unit abduction; both cases retain the same persistent referent
and return tokens of the same type. They separate two questions: which unit does
the current evidence concern, and, given a token for that unit, how does the
shared form respond to the supplied query? Direct access resolves the first
question without eliminating the second.

The first question is referential. A further structural-localization question
asks where the resolved or inferred referent lies in the token space used for
sharing information. A stable ID may answer the referential question without
supplying that position, just as a learned token geometry need not uniquely
identify the individual.

\paragraph{Direct unit access.}
A trusted resolver uses a key $k$ to determine which persistent unit the event
concerns and which stable address should be read. A separate lookup at that
address supplies the token $Z_{u(k)}^c$. With an ID-indexed model this is the
learned row
$z_\theta(k)$---an embedding, preference factor, or random effect---updated from
the records associated with that key. Personalized recommendation with a trusted user
key is a common example. Direct access therefore removes which-unit uncertainty
while leaving token-and-form approximation error. A history-dependent response
state may also be updated for the known unit, but such state belongs to
the tokenizer's context or to $R_\theta$ rather than to unit attribution. A
Dirac measure can record the resolved identity in common measure-valued
notation without changing the token's origin or turning lookup into
abduction; details are recorded in the appendix.

\paragraph{Unit abduction.}
Unit abduction applies when no resolver identifies the realized persistent
unit. Let $\mathcal O$ denote the factual evidence available before the
current answer. The access step forms a contextual unit token
\begin{equation}
  Z^c
  =
  T_\phi(\mathcal O,c)
  \label{eq:unit-abduction}
\end{equation}
of the same type as a lookup token. Alternative queries then use that token
through $R_\theta$. This yields the learner's token-space composition in
\cref{eq:world-learner-response-coordinates}.

The theoretically exact predictive law remains the mixture over units. Under
the data-generating distribution,
$P(U\in du\mid\mathcal O)$ is the world-side identity posterior: conditioning
changes uncertainty about which unit was realized, while the referent itself
remains fixed. The theorems below use an identity-mixture specialization and write
the learner's approximation as
\begin{equation}
  Q_\phi(du\mid\mathcal O).
  \label{eq:which-unit-belief}
\end{equation}
When the learning protocol explicitly targets the identity posterior, $Q_\phi$
may be interpreted and evaluated as an approximate posterior; otherwise its
calibration and recovery properties require separate targets and diagnostics.
A tokenizer may emit a token, or a law over tokens, without first forming an
identity posterior; \cref{eq:token-composition} gives the corresponding
prediction.
\paragraph{Unit information rather than a mandated belief.}
\label{subsec:unit-info-object}
The formed token $Z^c$ is not the unit and is not required to be a probability
over $\cU$. Direct access and abduction differ by how the token is obtained,
not by whether that information is a belief.
Conditioning on an unobserved quantity and averaging
predictions over its uncertainty are standard probabilistic-learning operations
\citep{bishop2006pattern,murphy2012machine}. The contribution claimed here is the
persistent-referent formulation and the separation of learner-access modes, not
those operations themselves. Established
causal and logical uses of \emph{abduction} are compared in \cref{sec:related}.

The unit boundary is declared before learning and evaluation. Within a declared
same-unit query family, the same realized value $u$ is retained across response
queries. A new factual observation may update the formed token; alternative
queries in the same family reuse the current token.

\begin{figure}[t]
  \centering
  \input{figures/fig2-belief-queries}
  \caption{\textbf{Direct unit access and unit abduction return tokens of the
same type.}
  A trusted resolver uses a key $k$ to fix the persistent referent and its
  stable address; lookup then supplies $Z_{u(k)}^c$, typically the learned row
  $z_\theta(k)$, without
  forming a which-unit belief. When no resolver
  identifies the unit, a tokenizer forms a contextual token $Z^c$ from
  factual evidence. A which-unit belief $Q_\phi$ may approximate the world-side
  identity conditional (\cref{app:interface-distinctions}). Both access modes
  use the same learned response model, with the token mode using the shared form
  $R_\theta$.}
  \label{fig:belief-queries}
\end{figure}
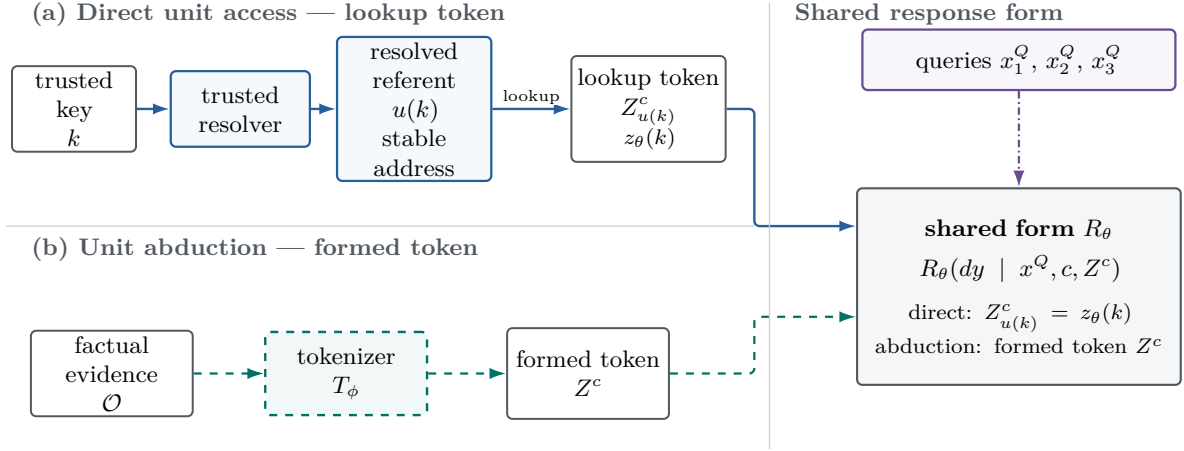

In the supervised specialization, an identity-mixture learner mirrors the
world-side unit mixture in \cref{eq:world-learner-response-coordinates} by
reusing one learned identity mixing measure through an identity-indexed response
model:
\begin{equation}
  \boxed{
  \widehat P_{\theta,\phi}(dy\mid\mathcal O;\,x^Q)
  :=\int_{\cU}P_\theta(dy\mid x^Q,u)
                  Q_\phi(du\mid\mathcal O).}
  \label{eq:core-composition}
\end{equation}
This identity-mixture form is used in the theorems below. A different learner
realization is the token-space form
\begin{equation}
  \widehat P_{\theta,\phi}(dy\mid\mathcal O;\,x^Q,c)
  :=
  \int
  R_\theta(dy\mid x^Q,c,z)\,
  T_\phi(dz\mid\mathcal O,c).
  \label{eq:token-composition}
\end{equation}
The two integrals coincide with each other and with the exact world mixture
only under the conditions of
\cref{eq:world-learner-response-coordinates}. The
semicolon distinguishes the supplied response query $x^Q$ from abductive
evidence. The displayed learner prediction is always well-defined when its
components are. To identify the corresponding exact fixed-belief mixture with
$P(Y^Q\in dy\mid\mathcal O,x^Q)$, we use two separate conditions:
response sufficiency,
$Y^Q\perp\!\!\!\perp\mathcal O\mid(U,x^Q)$, and the external-query condition,
$U\perp\!\!\!\perp X^Q\mid\mathcal O$ (or the same statement for the complete
input $(X^Q,C^Q)$ when context is exposed). The latter holds for a fixed query
and, more generally, when a query policy uses the declared factual evidence but
has no hidden access to $U$. If a recorded variable is informative about the unit,
the protocol places that factual information in $\mathcal O$ before forming
the token; an alternative response query does not silently become a new
abduction call. When answer-time conditions must be distinguished from the
focal query, we write the complete response input as $(x^Q,c^Q)$ and call
$c^Q$ \emph{response context}. It may otherwise be absorbed into a larger
$x^Q$.
Attribution evidence remains in $\mathcal O$, which excludes the current
prediction target.

\paragraph{The uninformative limit.}
If $U\perp\!\!\!\perp\mathcal O$, then
$P(U\in du\mid\mathcal O)=\Pi(du)$ almost surely.
Evidence-based individualization is then unavailable: a learner calibrated to
this target remains at population-level uncertainty, or an operational system
may abstain and seek new evidence. This limitation concerns attribution
information and is compatible with either homogeneous or heterogeneous
responses across units.

\paragraph{Evidence informativeness and response heterogeneity.}
Evidence informativeness asks whether the available evidence changes what can
be known about $U$. Response heterogeneity asks whether changing $U$ changes
the current response law. They define separate axes of the problem.
\begin{table}[t]
\centering
\caption{For a unit boundary fixed by the task, learner access to the unit and
unit-response dependence are independent axes. ``Current evidence'' means the
evidence declared by the protocol.}
\label{tab:attribution-response-axes}
\small
\begin{tabular}{@{}p{0.22\linewidth}p{0.32\linewidth}p{0.36\linewidth}@{}}
\toprule
Access to $U$ & unit-homogeneous response & unit-heterogeneous response \\
\midrule
resolved or informative & The referent may be resolved directly, or the
formed token may be informative; this response target is invariant across
units. &
Personalization is substantive; a resolved identity or an informative token
changes the prediction. \\
unresolved; uninformative & Unit uncertainty remains at the
population level; the response is unit-invariant. & Oracle predictive value of
knowing $U$ may exist, but
the current evidence cannot access it; mix, abstain, or collect new evidence. \\
\bottomrule
\end{tabular}
\end{table}
Recommendation with a trusted resolver occupies the resolved-identity,
heterogeneous cell; the lookup token \(z_\theta(k)\) may nevertheless remain
learned.
The lower-right cell separates the predictive value of unit information from
the accessibility of attribution evidence; the upper-left cell separates attribution from
relevance to the response target. The regular-conditional and support qualifications
are recorded in
\cref{app:two-attribution-response-axes}.

\section{Unit Information: Predictive Value, Access, and Deployment}
\label{sec:consequences}

In this section, ``unit information'' refers to the predictive value of knowing
the realized unit, decomposed into oracle value, accessibility, learned
realization, and component identifiability. The proofs use the identity-mixture
learner of \cref{eq:core-composition}.

The central theoretical question is not merely whether a unit-conditioned
model can be written. Under \cref{hyp:shared-form-token} the learned object is
the pair $(T_\phi,R_\theta)$. The question is which part of the oracle
predictive value of unit information is available through declared evidence,
how much of that available value the pair attains, and what the resulting
marginal prediction can certify about its internal token-and-form
decomposition. The corresponding layers are distinct:
\begin{equation}
  \boxed{
  \text{oracle value}
  \neq \text{predictive accessibility}
  \neq \text{learned-pair realization}
  \neq \text{component identifiability}.}
  \label{eq:four-theory-layers}
\end{equation}
Here \emph{learnability} is reserved for whether some learning rule drives the
approximation penalty to zero as training information grows. The identities
below are classical specializations of conditional entropy, mutual information,
the relative-entropy chain rule, and data processing
\citep{cover2005elements,gneiting2007strictly}. They organize the four layers;
learnability and component identification require the additional conditions
listed in the appendix.

\subsection{Oracle Predictive Value, Accessibility, and Residual Information}

\paragraph{Question.}
If oracle unit information is useful, how much of that value is available from
the factual evidence that the protocol actually exposes? Write $W$ for the
complete declared response input, with $W=X^Q$ in the base response specification and
$W=(X^Q,C^Q)$ when optional response context is exposed. The letter $W$ is not
a unit token $Z_u^c$. Consider a declared
deployment law for $(U,\mathcal O,W,Y^Q)$ and assume response sufficiency,
\begin{equation}
  Y^Q\perp\!\!\!\perp\mathcal O\mid(U,W).
  \label{eq:value-access-response-sufficiency}
\end{equation}
When the relevant conditional densities exist with respect to a common outcome
reference measure and the log-loss risks are finite, write
\begin{align}
  R_W^\star&:=E[-\log p(Y^Q\mid W)],
  &R_{\mathcal O,W}^\star&:=E[-\log p(Y^Q\mid\mathcal O,W)],
  &R_{U,W}^\star&:=E[-\log p(Y^Q\mid U,W)].
  \label{eq:value-access-bayes-risks}
\end{align}

\begin{proposition}[Value--access decomposition under log loss]
\label{prop:value-access-decomposition}
Under \cref{eq:value-access-response-sufficiency},
\begin{equation}
  R_{U,W}^\star\leq R_{\mathcal O,W}^\star\leq R_W^\star,
  \label{eq:value-access-risk-ladder}
\end{equation}
and the oracle value decomposes exactly as
\begin{align}
  \underbrace{R_W^\star-R_{U,W}^\star}_{V_{\rm oracle}}
  &=
  \underbrace{R_W^\star-R_{\mathcal O,W}^\star}_{V_{\rm accessible}}
  +
  \underbrace{R_{\mathcal O,W}^\star-R_{U,W}^\star}_{V_{\rm residual}}
  \notag\\
  &=I(U;Y^Q\mid W)
  \label{eq:main-oracle-unit-value}\\
  &=I(\mathcal O;Y^Q\mid W)
    +I(U;Y^Q\mid\mathcal O,W).
  \label{eq:value-access-mi-decomposition}
\end{align}
Moreover,
\begin{equation}
  I(\mathcal O;Y^Q\mid W)
  \leq
  \min\!\left\{I(U;\mathcal O\mid W),I(U;Y^Q\mid W)\right\}.
  \label{eq:value-access-data-processing}
\end{equation}
\end{proposition}

\paragraph{Intuition.}
Perfect knowledge of $U$ defines an oracle ceiling, but evidence can expose only
the response-relevant part of that information. The remainder is still
predictively valuable in principle and inaccessible under the declared evidence
cutoff.

The three gaps vanish exactly under $Y^Q\perp\!\!\!\perp U\mid W$,
$Y^Q\perp\!\!\!\perp\mathcal O\mid W$, and
$Y^Q\perp\!\!\!\perp U\mid(\mathcal O,W)$. The last says that the evidence
retains all response-relevant unit information, not that it recovers the
realized unit. A trusted key that uniquely resolves $U=u(k)$ zeros the residual
term; estimating $z_\theta(k)$ and $R_\theta$ remains the approximation problem
below. The proper-score comparison is
\cref{prop:proper-score-heterogeneity-gap}; proofs are in
\cref{app:value-access-proofs}.

\paragraph{Binary example.}
Let $W$ be fixed, $U\sim\operatorname{Bernoulli}(1/2)$, $Y^Q=U$, and
$\mathcal O=U\oplus N$ with
$N\sim\operatorname{Bernoulli}(\delta)$ independent and
$0\leq\delta\leq1/2$. In nats, the oracle value is $\log 2$, the accessible
value is $\log 2-h(\delta)$, and the residual value is $h(\delta)$, where
$h(\delta)=-\delta\log\delta-(1-\delta)\log(1-\delta)$. The unit matters for
every $\delta$, but the declared evidence ranges from perfectly informative to
useless.

\subsection{Realization by the Learned Pair and End-to-End Error}

\paragraph{Question.}
How do unit-belief and fixed-unit response errors limit the value realized by
the deployed mixture? For the fixed-belief formulation, additionally assume the
external-query condition
\begin{equation}
  U\perp\!\!\!\perp W\mid\mathcal O.
  \label{eq:external-query-contract}
\end{equation}
This holds for a fixed query and for a query policy that uses
$\mathcal O$ without hidden access to $U$. At admissible $(o,w)$, write
\begin{align}
  P_o(du)&:=P(U\in du\mid\mathcal O=o),
  &Q_o(du)&:=Q_\phi(du\mid o),\notag\\
  K_{u,w}(dy)&:=P(Y^Q\in dy\mid U=u,W=w),
  &\widehat K_{u,w}(dy)&:=P_\theta(dy\mid w,u),
  \label{eq:query-local-kernels}\\
  M_{o,w}(dy)&:=\int K_{u,w}(dy)P_o(du),
  &\widehat M_{o,w}(dy)&:=\int\widehat K_{u,w}(dy)Q_o(du).
  \label{eq:true-learned-response-mixtures}
\end{align}
Under \cref{eq:value-access-response-sufficiency,eq:external-query-contract},
$M_{o,w}=P(Y^Q\in\cdot\mid\mathcal O=o,W=w)$ on the chosen versions.

\begin{proposition}[End-to-end predictive excess and component-error bound]
\label{prop:end-to-end-approximation}
Let $R_{\rm dep}(\theta,\phi)$ be the expected log loss of
$\widehat M_{\mathcal O,W}$. Whenever the displayed quantities are finite,
\begin{align}
  R_{\rm dep}(\theta,\phi)-R_{\mathcal O,W}^\star
  &=E\!\left[D_{\rm KL}\!\left(M_{\mathcal O,W}
       \,\middle\|\,\widehat M_{\mathcal O,W}\right)\right]
   =:\mathcal E_{\rm pred}(\theta,\phi),
  \label{eq:learner-predictive-kl}\\
  R_W^\star-R_{\rm dep}(\theta,\phi)
  &=I(\mathcal O;Y^Q\mid W)-\mathcal E_{\rm pred}(\theta,\phi),
  \label{eq:learner-achieved-value}\\
  R_{\rm dep}(\theta,\phi)-R_{U,W}^\star
  &=I(U;Y^Q\mid\mathcal O,W)+\mathcal E_{\rm pred}(\theta,\phi).
  \label{eq:learner-oracle-decomposition}
\end{align}
If, for almost every $(o,w)$ under the deployment law of
$(\mathcal O,W)$, $P_o\ll Q_o$ and
$K_{u,w}\ll\widehat K_{u,w}$ for $P_o$-almost every $u$, then
\begin{align}
  \mathcal E_{\rm pred}(\theta,\phi)
  &\leq E_{\mathcal O}\!\left[
    D_{\rm KL}\!\left(P(U\in\cdot\mid\mathcal O)
    \,\middle\|\,Q_\phi(\cdot\mid\mathcal O)\right)\right]
  \notag\\
  &\quad+E_{\mathcal O,W}\!\left[
    \int D_{\rm KL}\!\left(K_{u,W}\,\middle\|\,\widehat K_{u,W}\right)
    P(du\mid\mathcal O)\right].
  \label{eq:latent-joint-kl-deployment}
\end{align}
The bound remains valid in the extended sense when its right-hand side is
infinite.
\end{proposition}

The exact identities say that the learned pair realizes accessible
response value minus its predictive excess. The upper bound follows by applying
the relative-entropy chain rule to the true and learned latent-joint laws and
then projecting away $U$; the proof is in
\cref{app:end-to-end-approximation-proof}.

For a complementary query-local statement, let
$d_{\rm TV}(\mu,\nu)=\sup_A|\mu(A)-\nu(A)|$ and let
$\mathcal U_o$ be a common full-mass region under $P_o$ and $Q_o$.

\begin{lemma}[Query-local unit-mixture error propagation]
\label{prop:mixture-stability}
For measurable response kernels for which the displayed quantities are
well-defined,
\begin{align}
  d_{\rm TV}\!\left(\widehat M_{o,w},M_{o,w}\right)
  &\leq
  \int d_{\rm TV}\!\left(\widehat K_{u,w},K_{u,w}\right)Q_o(du)
  \notag\\
  &\quad+
  \left[\sup_{u,v\in\mathcal U_o}
  d_{\rm TV}\!\left(K_{u,w},K_{v,w}\right)\right]
  d_{\rm TV}(Q_o,P_o).
  \label{eq:mixture-stability}
\end{align}
\end{lemma}

\paragraph{Intuition.}
The KL result tracks the log-loss penalty of the deployed prediction; the
total-variation result shows when unit-belief error can matter at a particular
query. If all true fixed-unit response laws agree there, the heterogeneity
diameter is zero and belief error cannot change the response mixture.

The KL bound splits deployed excess into unit-belief error and
true-posterior-averaged response-law error
(\cref{eq:latent-joint-kl-deployment}); the total-variation lemma makes belief
sensitivity heterogeneity-dependent
(\cref{app:mixture-stability-proof}). Under trusted direct access,
$P_o=Q_o=\delta_{u(k)}$ with no $Q_\phi$ module, so remaining error is mismatch
between $\widehat K_{u(k),w}$ and $K_{u(k),w}$---token and form, not
attribution. Both bounds are one-sided: a perfect marginal identifies neither
component, support mismatch can make the KL bound infinite, and total variation
does not control unbounded log loss.

\paragraph{Binary example and TV sensitivity.}
In the binary example above, the exact evidence-conditional predictor
has $\mathcal E_{\rm pred}=0$ and realizes all accessible value. A learner that
ignores $\mathcal O$ and always reports the pooled
$\operatorname{Bernoulli}(1/2)$ law has
$\mathcal E_{\rm pred}=I(\mathcal O;Y^Q)$ and realizes none of it. Separately,
if $K_{0,w}=K_{1,w}$ at a query, even a belief that swaps all mass between the
two units leaves the response prediction unchanged; when the laws separate, the
heterogeneity multiplier in \cref{eq:mixture-stability} records the possible
sensitivity.

\subsection{Single-Row Impossibility and Repeated-Linkage Separation}

\paragraph{Question.}
Can perfect marginal prediction validate the internal unit structure? The law
of total probability gives
\begin{equation}
  \int_{\cU}P_u^\star(dy\mid x)P(U\in du\mid X=x)
  =P(Y\in dy\mid X=x)
  \quad\text{for $P_X$-almost every $x$}.
  \label{eq:exact-rowwise-collapse}
\end{equation}
The next result upgrades this identity to an observational-equivalence and
testing statement, then exhibits a protocol that separates the same worlds.

\begin{proposition}[Single-row indistinguishability and linked-pair separation]
\label{prop:single-row-collapse}
\begin{enumerate}[label=(\alph*),leftmargin=*,itemsep=3pt]
  \item The identity in \eqref{eq:exact-rowwise-collapse} holds for compatible
  conditional versions. Moreover, in an unrestricted response-kernel class,
  every sample-only law $R(dy\mid x)$ has, under every learner belief, the
  unit-constant representation
  \begin{equation}
    K_u^R(dy\mid x):=R(dy\mid x),
    \label{eq:unit-constant-representation}
  \end{equation}
  so marginal fit does not identify the two mixture components.

  \item For a statistical witness, fix one response input. In world
  $\mathsf H$, let $U\sim\operatorname{Bernoulli}(1/2)$ and
  \begin{equation}
    P_{\mathsf H}(Y=1\mid U=0)=\tfrac14,
    \qquad
    P_{\mathsf H}(Y=1\mid U=1)=\tfrac34.
    \label{eq:bernoulli-heterogeneous-world}
  \end{equation}
  In world $\mathsf P$, let
  $P_{\mathsf P}(Y=1\mid U=u)=1/2$ for both units. If each observation comes
  from a fresh independently drawn unit and neither its identity nor linkage
  is observed, then for every $n$,
  \begin{equation}
    \mathcal L_{\mathsf H}(Y_1,\ldots,Y_n)
    =\mathcal L_{\mathsf P}(Y_1,\ldots,Y_n)
    =\operatorname{Bernoulli}(1/2)^{\otimes n}.
    \label{eq:single-row-equal-laws}
  \end{equation}
  Consequently every possibly randomized test has
  $\alpha_n+\beta_n=1$ and
  \begin{equation}
    \inf_{\varphi_n}\max\{\alpha_n,\beta_n\}=\tfrac12.
    \label{eq:single-row-minimax-error}
  \end{equation}

  \item If instead two responses per unit are observed with trusted same-unit
  linkage and are conditionally independent given $U$, then
  \begin{equation}
    \operatorname{Cov}_{\mathsf H}(Y_1,Y_2)=\tfrac1{16},
    \qquad
    \operatorname{Cov}_{\mathsf P}(Y_1,Y_2)=0,
    \label{eq:linked-pair-covariance}
  \end{equation}
  or equivalently their agreement probabilities are $5/8$ and $1/2$.
  If $M$ linked pairs are sampled independently from independently drawn units,
  their empirical agreement rate therefore gives a consistent test as
  $M\to\infty$.

  \item More generally, under the declared linked-pair model
  $(Y_1,Y_2)\mid U=u\sim\operatorname{Bernoulli}(p_u)^{\otimes2}$, the
  observable pair law identifies
  \begin{equation}
    \operatorname{Var}(p_U)
    =P(Y_1=1,Y_2=1)-P(Y_1=1)^2
    =\operatorname{Cov}(Y_1,Y_2),
    \label{eq:linked-pair-variance-identification}
  \end{equation}
  but does not by itself identify the full mixing law or realized unit.
\end{enumerate}
\end{proposition}

\paragraph{Intuition.}
One response from each fresh, unlinked unit reveals only the pooled Bernoulli
mean. Trusted repeated linkage exposes a joint observation, whose within-unit
dependence carries a heterogeneity signal.

Unrestricted single-row observations cannot distinguish a homogeneous world
from every heterogeneous alternative. Trusted same-unit pairs under a
fixed-$u$ product law separate the displayed worlds and identify
$\operatorname{Var}(p_U)$. Repeated records without trusted linkage do not
suffice; recommendation logs with user IDs are outside the fresh-unlinked
regime because the IDs already supply linkage and update addresses. Restricted
mixture identifiability can require further structure
\citep{teicher1963identifiability}. Proof and the explicit test are in
\cref{app:single-row-linked-proof}.

\paragraph{Binary example.}
The two fixed-unit success probabilities are $1/2\pm1/4$, so their population
variance is $(1/4)^2=1/16$, exactly the linked-pair covariance. Both probabilities
remain strictly between zero and one: the witness retains event variation after
$U$ is fixed rather than replacing it by a deterministic label.

\begin{figure}[t]
  \centering
  \input{figures/fig3-marginal-collapse}
  \caption{\textbf{Single-row marginals do not certify unit structure.}
  Distinct fixed-unit Bernoulli laws and a unit-independent Bernoulli law can
  both yield $P(Y=1)=1/2$ after marginalizing $U$. In the restricted witness of
  \cref{prop:single-row-collapse}, one response from each fresh unit remains
  indistinguishable for every sample size, whereas a conditionally independent
  pair with trusted same-unit linkage reveals within-unit covariance.}
  \label{fig:marginal-collapse}
\end{figure}
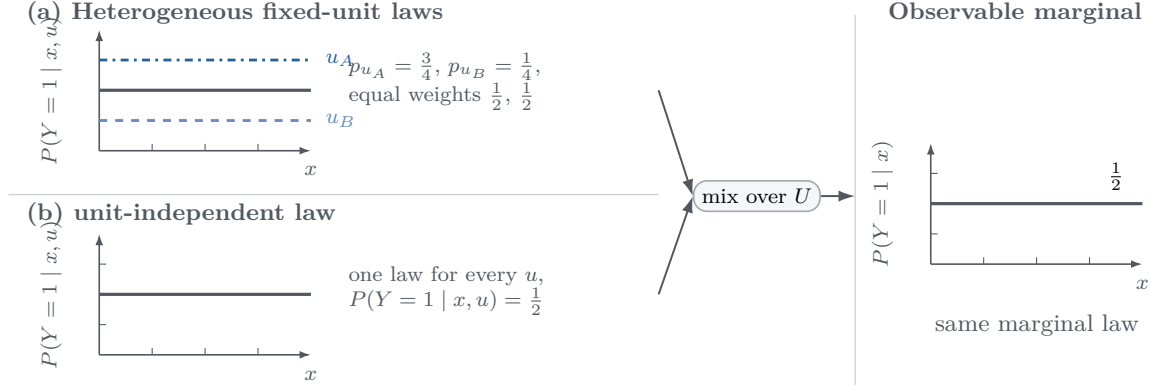

\subsection{Practical Corollaries for Repeated-Unit Evaluation}
\label{sec:framework-consequences}

When attribution is observed, retaining units also changes elementary
bookkeeping in ways that alter the estimand and evaluation protocol
\citep{hurlbert1984pseudoreplication,saeb2017need,roberts2017crossvalidation}.
For $M$ observed units, let $m(u)$ be the number of records from $u$ and let
$\bar\ell(u)$ be its mean loss.

\begin{corollary}[Row and unit weighting]
\label{prop:risk-weighting}
The empirical risks
\begin{equation}
  \widehat R_{\rm row}=\sum_u\frac{m(u)}{N}\bar\ell(u),
  \qquad
  \widehat R_{\rm unit}=\frac1M\sum_u\bar\ell(u)
  \label{eq:main-row-unit-risk}
\end{equation}
agree for every possible collection of unit mean losses if and only if all
observed units have equal record multiplicity.
\end{corollary}

\begin{corollary}[Known units survive a record-wise split]
\label{prop:split}
If each record independently enters training with probability $p\in(0,1)$, a
unit with $m(u)$ records appears in both train and test with probability
\begin{equation}
  1-p^{m(u)}-(1-p)^{m(u)}.
  \label{eq:main-record-split-overlap}
\end{equation}
This probability is positive exactly when $m(u)\geq2$.
\end{corollary}

Neither weighting is universally correct; the deployment estimand determines
the target. Likewise, a record-wise split generally mixes new-event prediction
for known units with prediction for unseen units, so an unseen-unit claim
requires an explicitly unit-disjoint protocol. Proofs are in
\cref{app:repeated-unit-evaluation-proofs}. Prediction-level marginalization is
distinct from deleting attribution metadata: removing unit labels from a table
is a data map, not the same operation as mixing a unit-conditioned family.

\begin{figure}[t]
  \centering
  \input{figures/fig4-repeated-units}
  \caption{\textbf{Repeated units change both the estimand and the split
  question.}
  Row averaging and unit averaging assign different weights when units
  contribute unequal numbers of records. A record-wise split can measure a
  new event for a known unit, whereas a whole-unit split targets
  generalization to a new unit. The declared deployment question determines the
  appropriate choice.}
  \label{fig:repeated-units}
\end{figure}
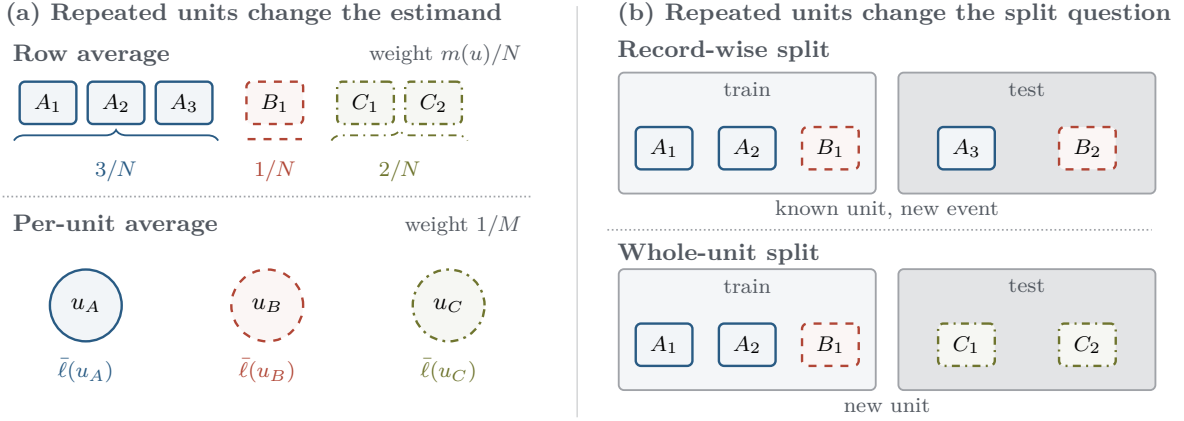

An empirical study that uses the identity-mixture specialization should keep
separate unit-belief quality, fixed-unit response quality, deployed
marginalized prediction, and a matched unit-omitting baseline under the same
information and target.
A compact checklist is in \cref{app:evaluation-checklist}.
\section{Recommendation as a Worked Setting}
\label{sec:recommendation}

Recommendation provides a worked instance of the unit primitive: many
interaction records can share one user while candidate items change. Direct
unit access uses a trusted key to resolve which persistent user is active and
which stable address to read; lookup supplies \(Z_{u(k)}^c=z_\theta(k)\). Unit abduction forms a
token of the same type from factual evidence
(\cref{sec:unit-access}). The two access modes differ in how the token is obtained,
not in whether the user is the unit.
The mapping is summarized in
\cref{tab:recommendation-map}.
Factorization and neural collaborative-filtering models are established
instances of the ID-indexed formulation
\citep{mnih2007pmf,koren2009matrix,he2017neural}.
In the linear running instance they instantiate
\cref{eq:linear-predictor} by a user token \(z_\theta(k)\) and an item map
\(\psi_\theta(x^Q)\); a deeper query encoder may be nonlinear, but the readout
in the user token remains an inner product.
\begin{table}[t]
\centering
\caption{For this worked task, users instantiate the declared unit boundary.}
\label{tab:recommendation-map}
\small
\begin{tabular}{@{}p{0.30\linewidth}p{0.60\linewidth}@{}}
\toprule
Object in the unit formulation & Recommendation interpretation \\
\midrule
\multicolumn{2}{@{}l}{\textbf{Population}} \\
population unit $U$ & possible user individuals \\
\multicolumn{2}{@{}l}{\textbf{Individual referent}} \\
realized unit $u$ & one persistent user; not a uniquely identified embedding coordinate \\
\multicolumn{2}{@{}l}{\textbf{Direct-access model}} \\
 trusted key $k$ & stable referential address that associates interaction records and parameter updates with the same user slot \\
lookup token $Z_{u(k)}^c$ & learner-side object obtained by indexing, typically represented by a learned row $z_\theta(k)$ (embedding, factor, random effect); not the user $u(k)$ \\
\multicolumn{2}{@{}l}{\textbf{Learner under unit abduction}} \\
factual evidence $\mathcal O$ & admissible observed history or side information used to form a user token \\
tokenizer $T_\phi$ & maps evidence and context to a user token $Z^c$; when explicitly targeted, $Q_\phi$ approximates the world-side identity conditional for the identity-mixture learner \\
\multicolumn{2}{@{}l}{\textbf{Query--response specification}} \\
query $x^Q$ & candidate item or slate supplied to the response model \\
response context $c^Q$ & explicit answer-time context or history-dependent response state, when used; it may change the token while user identity stays fixed \\
shared form $R_\theta(dy\mid x^Q,c^Q,Z)$ & feedback law instantiated by the user token through a simple relation in $Z$; the linear running instance is $\eta=\alpha_\theta(x^Q,c^Q)+\langle\psi_\theta(x^Q,c^Q),Z\rangle$; a direct-ID model computes it at $Z_{u(k)}^c=z_\theta(k)$ \\
\bottomrule
\end{tabular}
\end{table}

\section{Relations to Established Learning Traditions}
\label{sec:related}

Classical statistical learning studies how a predictor class and learning rule
generalize from finite samples
\citep{hastie2009elements,shalevshwartz2014understanding}. The unit formulation
does not replace that theory. Under \cref{hyp:shared-form-token} it adds an
explicit question when persistent units matter: how a tokenizer organizes
units so that one shared response-law form generalizes across them.

The components of the pair $(T_\phi,R_\theta)$ have long appeared across
statistics and machine learning under different objects and assumptions.
Longitudinal random-effects models provide a particularly close comparison
because they combine supplied same-subject linkage, a shared response form, and
a subject-specific token (the random effect). Other traditions infer
attribution, organize prediction around episodes or environments rather than
individuals, or begin from a supplied group boundary. \Cref{tab:related-families}
compares six traditions by their persistent referent, attribution assumptions,
identity uncertainty, and how they instantiate a tokenizer and a shared form.
The table records a recurring structural role rather than mathematical
equivalence: each tradition retains its own estimands, identification
conditions, and learnability assumptions.

\begin{table}[t]
\centering
\caption{Relations to six established model families.}
\label{tab:related-families}
\footnotesize
\begin{tabular}{@{}>{\raggedright\arraybackslash}p{0.20\linewidth}
                    >{\raggedright\arraybackslash}p{0.35\linewidth}
                    >{\raggedright\arraybackslash}p{0.35\linewidth}@{}}
\toprule
Family & Established object or assumption & Relation to the unit formulation \\
\midrule
random effects & A supplied subject or group index and a population model for
subject-specific parameters \citep{laird1982random,gelman2006data}. & Direct
lookup of a subject token with a shared $R_\theta$; the random effect is a
distributional token, not a private response mechanism. \\
latent-variable models and mixtures of experts & Incomplete-data estimation,
input-dependent gating over experts, and amortized inference over latent codes
\citep{dempster1977maximum,jacobs1991adaptive,kingma2014autoencoding}. & The
encoder or gating map is a tokenizer; the decoder or expert family is a shared
form, provided the task assigns the latent the role of a persistent unit. \\
multitask, grouped-latent, Neural-Process, and meta-learning methods & Shared
parameters or representations across related tasks, and dataset-, group-, or
task-level latents or deterministic context representations derived from an
already associated context set
\citep{caruana1997multitask,baxter2000model,edwards2017neural,
bouchacourt2018mlvae,garnelo2018neural,
garnelo2018conditional,gordon2019versa}. & An episode or context encoder is a
tokenizer and the decoder is a shared form, provided the outer task declares a
persistent unit. In the cited methods the context-to-group association is
supplied before prediction, which differs from forming a token from
otherwise unattributed evidence. \\
record linkage, entity resolution, and probabilistic identity representations
& Posterior linkage structures, learned entity matching, and
distribution-valued instance representations
\citep{steorts2016entity,steorts2018generalized,kaplan2022practical,
li2020ditto,oh2019modeling,shi2019probabilistic}. & These methods infer which
records share a task-side unit. The present formulation additionally asks for
a learner-side token that instantiates one shared $R_\theta$. \\
recommenders and personalization & Known user indices, latent user factors,
history encoders, and user-conditioned response scores
\citep{mnih2007pmf,koren2009matrix,salakhutdinov2008bpmf,he2017neural,
kang2018sasrec,liang2018multvae}. & A known user index provides a lookup token
$z_\theta(k)$ while the scoring model is a shared $R_\theta$. Unresolved user
attribution forms a token of the same type rather than looking it up. \\
potential outcomes, structural causal models, and Abductive Learning & units
under treatment comparisons, counterfactual
abduction--intervention--prediction, and logical abduction coupled to learning
\citep{rubin1974causal,pearl2009causality,pawlowski2020dscm,
dai2019bridging}. & Here $u$ remains a task-side referent across treatments or
worlds. Their abduction objects (noise, logical facts, identity posteriors)
are not the unit token of \cref{hyp:shared-form-token}, and each tradition
keeps its own estimands and identification conditions. \\
\bottomrule
\end{tabular}
\end{table}

The history also separates persistent-unit variation from environment
variation. Random effects and personalization primarily ask how responses vary
across persistent individuals. Covariate shift, domain adaptation, and
invariant prediction primarily organize variation across environments or
selection regimes \citep{shimodaira2000covariate,bendavid2010theory,
peters2016invariant}. These axes can be crossed: a response may vary across
units while remaining invariant across environments, vary across environments
while remaining homogeneous across units, or vary along both. Collapsing these
axes would conflate distinct sources of variation.

The term \emph{abduction} likewise has established neighboring meanings. In
structural-causal counterfactual reasoning it infers exogenous noise before
intervention and prediction, while Abductive Learning couples logical
abduction to machine learning \citep{pearl2009causality,pawlowski2020dscm,
dai2019bridging}. Unit abduction instead names the construction of a contextual
unit token from factual evidence (\cref{eq:unit-abduction}). A which-unit
belief may approximate the world-side identity conditional and feed the
identity-mixture learner; the access step remains the formed token
(\cref{app:interface-distinctions}). These
learner formulations may coexist in a larger model while remaining distinct
inference operations.

\section{Discussion and Limitations}
\label{sec:discussion}

The unit primitive is useful when a learning question depends on persistence,
attribution, or unit-level evaluation. For row-level deployments with no
same-unit question, the ordinary sample formulation is the appropriate special
case. That sufficiency does not make the rows homogeneous units: without a
declared unit axis the row-level law is silent about unit-response homogeneity.

Two readings of the preceding sections are methodological rather than theorems.
For many applied tasks, a natural starting class admits structured response
heterogeneity and lets data decide whether the unit-omitting restriction
\cref{eq:constant-token,eq:token-ignored} is adequate; this does not claim that
every task is empirically heterogeneous. Separately, the comparisons in
\cref{sec:related} can be read as an organizing program in which different
fields recover partial aspects of a unit-conditioned family. That reading is
not a reduction.

\Cref{hyp:shared-form-token} is a computational interface, not a
learnability theorem. The structured class makes it a restriction: $R_\theta$
must read the token through a declared simple relation. The running instance
\cref{def:linear-shared-form} is a finite-dimensional linear predictor; other
declared simple maps of $z$ instantiate the same class. Without a bound on $d$,
or with an unrestricted nonlinear readout of $z$, the
factorization is again a reparameterization. The class is misspecified when
residual unit-specific variation cannot be absorbed into the declared simple
relation. The paper does not prove that some learning rule drives the
approximation penalty to zero as training information grows, and joint training
of $(T_\phi,R_\theta)$ does not by itself identify either module. The theorems
in \cref{sec:consequences} analyze an identity-mixture learner; a deployed
system may instead use the token composition. These are alternative
coordinates, not a proof that a tokenizer recovers the identity posterior.

The unit boundary is task-relative. Within a fixed declaration, a trusted key
resolves the referent but does not identify the token: $z_\theta(k)$ may be
reparameterized with $R_\theta$, and learning it does not recover the
individual. Referential identity also supplies no intrinsic similarity;
borrowing requires a metric, kernel, graph, hierarchy, or other coupling. In the
linear running instance the inner product in \cref{eq:linear-predictor} is the
default readout; other simple relations remain open. When identity is unresolved, identifying
$Q_\phi$ and the fixed-unit kernels still needs shared-unit observations or
other restrictions. Response-informative observations belong in $\mathcal O$;
a merely proposed query does not trigger a second abduction step.

The formal results already indicate what would test a unit-centered claim.
Unlinked single-row observations cannot separate a restricted heterogeneous
witness from a pooled homogeneous world; trusted same-unit pairs can
(\cref{prop:single-row-collapse}). A record-wise split mixes known-unit
new-event prediction with unseen-unit generalization
(\cref{prop:split}), and row weighting disagrees with unit weighting whenever
multiplicities differ (\cref{prop:risk-weighting}). An empirical study should
therefore keep a matched unit-omitting baseline, report unit-belief quality
separately from fixed-unit response quality when the identity-mixture
specialization is used, and declare whether the split is unit-disjoint
(\cref{app:evaluation-checklist}). No such study is reported here.

Causal estimands require an additional layer. The primitive preserves the
referent across a query family; intervention, counterfactual coupling, and
identification assumptions are specified separately
(\cref{app:discoscm-specialization}). Token reparameterization, the
identity-belief specialization, and the fixed evidence cutoff are collected in
\cref{app:interface-distinctions,app:implementation-boundaries}.

\section{Conclusion}

The paper's organizing claim is that machine learning learns shared structure
across task-declared units from noisy, selectively observed, unit-linked events.
A sample records an event; the task first declares the persistent referent, and
only then chooses an ID, embedding, or other representation. Supervised learning
takes a family of unit-conditioned response laws as its semantic object.
\Cref{hyp:shared-form-token} states a computational factorization for learning that
family: every learner-side path from the unit to prediction passes through a
contextual unit token, all units use one form $R_\theta$, and represented
unit-specific differences must pass through the token. The learned object is the
pair $(T_\phi,R_\theta)$; the structured class is a simple relation in the token,
with a linear predictor as the running instance
(\cref{def:linear-shared-form}). A
unit-omitting learner specification is unit-insensitive, while a constant or
ignored token is a canonical realization of that specification. The conceptual boundaries are
collected in \cref{app:interface-distinctions}.

Oracle predictive value of unit information is then separated from evidence
accessibility; the learned pair incurs an additional approximation penalty;
and marginal prediction alone does not reveal the internal token-and-form
decomposition. Trusted same-unit linkage provides a restricted positive
boundary. When preserving the unit changes none of a declared target, estimand,
admissible answer, or evaluation protocol, the unit-marginal formulation is
sufficient.

\begingroup
\footnotesize
\setlength{\bibsep}{0.2pt plus 0.1ex}
\bibliographystyle{plainnat}
\bibliography{references}
\endgroup

\clearpage
\appendix
\input{appendix}

\end{document}

%% file: figures/figure-style.tex
\definecolor{ualUnitBlue}{RGB}{39,94,156}
\definecolor{ualBeliefTeal}{RGB}{0,112,104}
\definecolor{ualEventAmber}{RGB}{190,104,8}
\definecolor{ualQueryPurple}{RGB}{104,75,145}
\definecolor{ualNeutral}{RGB}{82,88,96}
\definecolor{ualLight}{RGB}{244,246,248}
\definecolor{ualIdentityA}{RGB}{41,91,132}
\definecolor{ualIdentityB}{RGB}{177,70,54}
\definecolor{ualIdentityC}{RGB}{111,117,47}

\tikzset{
  ual flow/.style={-{Latex[length=2.2mm,width=1.5mm]}, line width=0.75pt, draw=ualNeutral},
  ual world flow/.style={-{Latex[length=2.2mm,width=1.5mm]}, line width=0.9pt, draw=ualUnitBlue},
  ual learner flow/.style={-{Latex[length=2.2mm,width=1.5mm]}, line width=0.85pt, dashed, draw=ualBeliefTeal},
  ual query flow/.style={-{Latex[length=2.2mm,width=1.5mm]}, line width=0.85pt, dash pattern=on 2.8pt off 1.3pt on 0.7pt off 1.3pt, draw=ualQueryPurple},
  ual prohibited flow/.style={-{Latex[length=2.0mm,width=1.4mm]}, line width=0.75pt, densely dotted, draw=ualNeutral!70},
  ual world/.style={draw=ualUnitBlue, fill=ualUnitBlue!4, rounded corners=2pt, line width=0.85pt, align=center, inner sep=5pt},
  ual unit/.style={circle, draw=ualUnitBlue, fill=ualUnitBlue!12, line width=1.05pt, align=center, minimum size=9.5mm, inner sep=1pt},
  ual event/.style={draw=ualNeutral, fill=white, rounded corners=2pt, line width=0.75pt, align=center, minimum height=8mm, inner sep=3.5pt},
  ual variation/.style={draw=ualEventAmber, fill=ualEventAmber!5, rounded corners=2pt, line width=0.8pt, align=center, inner sep=3.5pt},
  ual learner/.style={draw=ualBeliefTeal, fill=ualBeliefTeal!5, rounded corners=2pt, dashed, line width=0.9pt, align=center, inner sep=5pt},
  ual belief/.style={ellipse, draw=ualBeliefTeal, fill=ualBeliefTeal!6, dashed, line width=0.95pt, align=center, minimum width=28mm, minimum height=11mm, inner sep=3pt},
  ual query/.style={draw=ualQueryPurple, fill=ualQueryPurple!5, rounded corners=2pt, line width=0.85pt, align=center, inner sep=4pt},
  ual target/.style={draw=ualNeutral, fill=ualLight, rounded corners=2pt, densely dotted, line width=0.8pt, align=center, inner sep=4pt},
  ual output/.style={draw=ualNeutral, fill=ualLight, rounded corners=2pt, line width=0.8pt, align=center, inner sep=4pt},
  ual panel/.style={draw=ualNeutral!45, fill=white, rounded corners=3pt, line width=0.6pt, inner sep=6pt},
  ual panel divider/.style={draw=ualNeutral!28, line width=0.5pt},
  ual restriction/.style={draw=ualNeutral!65, fill=ualLight!55, rounded corners=3pt, dashed, line width=0.8pt},
  ual identity A/.style={draw=ualIdentityA, fill=ualIdentityA!6, line width=0.9pt},
  ual identity B/.style={draw=ualIdentityB, fill=ualIdentityB!5, dashed, line width=0.9pt},
  ual identity C/.style={draw=ualIdentityC, fill=ualIdentityC!6, dash dot, line width=0.9pt},
  ual panel title/.style={font=\footnotesize\bfseries, text=ualNeutral, anchor=west},
  ual note/.style={font=\footnotesize, text=ualNeutral, align=center},
  ual tiny note/.style={font=\scriptsize, text=ualNeutral, align=left},
  ual brace/.style={decorate, decoration={brace, amplitude=4pt}, line width=0.7pt, draw=ualNeutral},
  ual banner/.style={draw=ualNeutral!45, fill=ualLight, rounded corners=3pt, line width=0.6pt},
  ual pill/.style={draw=ualNeutral!50, fill=ualLight, rounded corners=5pt, line width=0.65pt, font=\scriptsize, inner sep=3pt, align=center},
  ual axis/.style={-{Latex[length=1.5mm,width=1.05mm]}, draw=ualNeutral, line width=0.6pt},
  ual split box/.style={draw=ualNeutral!55, rounded corners=2pt, line width=0.7pt},
  ual split train/.style={ual split box, fill=ualLight},
  ual split test/.style={ual split box, fill=ualNeutral!16},
  ual chip/.style={ual event, minimum width=7.5mm, minimum height=5.6mm, inner sep=0.7pt, font=\scriptsize},
  pics/ual plot axis/.style={
    code={
      \draw[ual axis] (0,0) -- (28,0);
      \draw[ual axis] (0,0) -- (0,16);
      \foreach \t in {7,14,21}
        \draw[draw=ualNeutral,line width=0.4pt] (\t,0) -- ++(0,0.9);
      \foreach \t in {4,8,12}
        \draw[draw=ualNeutral,line width=0.4pt] (0,\t) -- ++(0.9,0);
      \node[anchor=north,font=\scriptsize,text=ualNeutral] at (28,-0.7) {$x$};
    }
  }
}

\pgfplotsset{
  compat=1.18,
  ual response axis/.style={
    width=0.29\linewidth,
    height=31mm,
    axis lines=left,
    xmin=0, xmax=1,
    ymin=0, ymax=1,
    xtick=\empty,
    ytick=\empty,
    xlabel={$x$},
    ylabel={$\mathbb E[Y\mid x,u]$},
    label style={font=\footnotesize},
    title style={font=\footnotesize\bfseries, yshift=-1mm},
    axis line style={draw=ualNeutral},
    clip=false
  }
}

%% file: figures/fig1-events-unit.tex
\begin{tikzpicture}[x=1mm,y=1mm,font=\footnotesize]
  \path[use as bounding box] (0,0) rectangle (160,48);

  \path[ual banner] (0,41.6) rectangle (160,47.6);
  \draw[draw=ualNeutral!35,line width=0.6pt]
    (80,42.2) -- (80,47.0);
  \node[anchor=west,ual note,text width=72mm,align=left]
    at (3.0,44.6)
    {\textbf{Task declaration:} what counts as one unit};
  \node[anchor=west,ual note,text width=74mm,align=left]
    at (83.0,44.6)
    {\textbf{Population statement:} $U\sim\Pi\ \longrightarrow\ U=u$};

  \node[ual panel title] at (3.0,39.0)
    {(a) Events share persistent unit identity};

  \node[ual event,inner sep=3.2pt] (event1) at (20,31.4)
    {event 1\\$(x_1,y_1)$};
  \node[ual event,inner sep=3.2pt] (event3) at (20,20.2)
    {event 3\\$(x_3,y_3)$};
  \node[ual unit,minimum size=10mm] (unitA-left) at (56,25.8)
    {unit\\$u_A$};

  \node[ual event,inner sep=3.2pt] (event2) at (20,7.4)
    {event 2\\$(x_2,y_2)$};
  \node[ual unit,minimum size=10mm] (unitB-left) at (56,7.4)
    {unit\\$u_B$};

  \draw[ual world flow] (event1.east) -- (unitA-left.150);
  \draw[ual world flow] (event3.east) -- (unitA-left.210);
  \draw[ual world flow] (event2.east) -- (unitB-left.west);

  \draw[ual panel divider] (80,1.4) -- (80,40.2);
  \node[ual panel title] at (83.0,39.0)
    {(b) The same unit selects a response law};

  \node[ual event,inner sep=3.2pt] (sample-right) at (94,25.4)
    {event $i$\\$(x_i,y_i)$};
  \node[ual unit,minimum size=10mm] (unitA-right) at (116,25.4)
    {unit\\$u_A$};
  \node[ual world,text width=24mm,inner sep=3pt] (law) at (142,25.4)
    {response law\\$P^\star_{u_A}(dy\mid x)$};

  \draw[ual world flow] (sample-right.east) -- (unitA-right.west);
  \draw[ual world flow] (unitA-right.east) -- (law.west);

  \node[ual variation,text width=44mm] (remaining-variation) at (129,8.4)
    {$u_A$ selects the law, not the event};
  \draw[densely dotted,line width=0.75pt,draw=ualNeutral!70]
    (law.south) -- (remaining-variation.north);
\end{tikzpicture}

%% file: figures/fig-response-law-wedge.tex
\begin{tikzpicture}[font=\footnotesize,x=1mm,y=1mm]
  \path[use as bounding box] (0,0) rectangle (160,58);

  \node[ual panel title] at (2,54) {(a) Unit-insensitive learner};
  \foreach \y/\id/\lab in {39/A/{u_A},29/B/{u_B},19/C/{u_C}}
    \node[ual unit,minimum size=8mm,inner sep=0.4pt]
      (hom-\id) at (9,\y) {$\lab$};
  \node[ual world,text width=24mm,minimum height=18mm,inner sep=2pt]
    (common-law) at (36,29) {common token $z_0$\\[2pt]shared $R_\theta$};
  \draw[ual world flow] (hom-A.east) -- (common-law.west);
  \draw[ual world flow] (hom-B.east) -- (common-law.west);
  \draw[ual world flow] (hom-C.east) -- (common-law.west);

  \node[ual panel title] at (56,54) {(b) Shared form $R_\theta$};
  \node[font=\scriptsize,text=ualNeutral] at (86,46.2) {tokenizer $T_\phi$};
  \foreach \y/\id/\lab/\ztok in {
      39/A/{u_A}/{$Z_A^c$},
      29/B/{u_B}/{$Z_B^c$},
      19/C/{u_C}/{$Z_C^c$}} {
    \node[ual unit,minimum size=8mm,inner sep=0.4pt]
      (structured-\id) at (60,\y) {$\lab$};
    \node[ual event,text width=11mm,minimum height=7mm,inner sep=1pt]
      (tok-\id) at (86,\y) {\ztok};
    \draw[ual world flow] (structured-\id.east) -- (tok-\id.west);
  }
  \node[ual world,text width=15mm,minimum height=28mm,inner sep=2pt]
    (shared-R) at (114,29) {shared\\$R_\theta$};
  \draw[ual learner flow] (tok-A.east) -- (shared-R.west);
  \draw[ual learner flow] (tok-B.east) -- (shared-R.west);
  \draw[ual learner flow] (tok-C.east) -- (shared-R.west);

  \node[ual panel title] at (124,54) {(c) Saturated};
  \foreach \y/\id/\lab in {39/A/{u_A},29/B/{u_B},19/C/{u_C}} {
    \node[ual unit,minimum size=8mm,inner sep=0.4pt]
      (saturated-\id) at (128,\y) {$\lab$};
    \node[ual world,text width=14mm,minimum height=7mm,inner sep=1pt]
      (saturated-law-\id) at (152,\y) {$R_{\lab}$};
    \draw[ual world flow]
      (saturated-\id.east) -- (saturated-law-\id.west);
  }

  \draw[ual panel divider] (54,4) -- (54,56);
  \draw[ual panel divider] (122,4) -- (122,56);
\end{tikzpicture}

%% file: figures/fig2-belief-queries.tex
\begin{tikzpicture}[font=\footnotesize,x=1mm,y=1mm]
  \path[use as bounding box] (0,0) rectangle (160,64);

  \node[ual panel title] at (3,60.5)
    {(a) Direct unit access --- lookup token};
  \node[ual event,text width=15mm,minimum height=10mm,inner sep=2pt]
    (key) at (10,48) {trusted key\\$k$};
  \node[ual world,text width=17mm,minimum height=10mm,inner sep=2pt]
    (resolver) at (32,48) {trusted\\resolver};
  \node[ual world,text width=19mm,minimum height=12mm,inner sep=2pt]
    (address) at (55,48) {resolved referent\\$u(k)$\\stable address};
  \node[ual event,text width=19mm,minimum height=14mm,inner sep=2pt]
    (resolved) at (86,48) {lookup token\\$Z_{u(k)}^c$\\{\scriptsize $z_\theta(k)$}};
  \draw[ual world flow] (key) -- (resolver);
  \draw[ual world flow] (resolver) -- (address);
  \draw[ual world flow] (address) -- node[above=1pt,font=\tiny,fill=white,inner sep=0.5pt]{lookup} (resolved);

  \node[ual panel title] at (3,29.5)
    {(b) Unit abduction --- formed token};
  \node[ual event,text width=20mm,minimum height=11mm,inner sep=2pt]
    (evidence) at (15,13) {factual evidence\\$\mathcal O$};
  \node[ual learner,text width=20mm,minimum height=11mm,inner sep=2pt]
    (abduction) at (46,13) {tokenizer\\$T_\phi$};
  \node[ual event,text width=20mm,minimum height=12mm,inner sep=2pt]
    (token) at (78,13)
    {formed token\\$Z^c$};
  \draw[ual learner flow] (evidence) -- (abduction);
  \draw[ual learner flow] (abduction) -- (token);

  \node[ual panel title] at (104,60.5) {Shared response form};
  \node[ual query,text width=40mm,minimum height=8mm,inner sep=2pt]
    (queries) at (135,54.5)
    {queries $x^Q_1,\,x^Q_2,\,x^Q_3$};
  \node[ual output,text width=40mm,minimum height=26mm,inner sep=4pt,align=center]
    (response) at (135,24.5)
    {{\footnotesize\bfseries shared form $R_\theta$}\\[4pt]
     $R_\theta(dy\mid x^Q,c,Z^c)$\\[5pt]
     {\scriptsize direct: $Z_{u(k)}^c=z_\theta(k)$}\\[2pt]
     {\scriptsize abduction: formed token $Z^c$}};

  \draw[ual query flow] (queries) -- (response.north);
  \draw[ual world flow,rounded corners=2pt]
    (resolved.east) -- (100,48) |- ([yshift=8mm]response.west);
  \draw[ual learner flow,rounded corners=2pt]
    (token.east) -- (100,13) |- ([yshift=-4mm]response.west);

  \draw[ual panel divider] (1,32.5) -- (100,32.5);
  \draw[ual panel divider] (102,3) -- (102,62);
\end{tikzpicture}

%% file: figures/fig3-marginal-collapse.tex
\begin{tikzpicture}[
  font=\footnotesize,x=1mm,y=1mm,
  unit A/.style={draw=ualUnitBlue,line width=1.05pt,dash dot},
  unit B/.style={draw=ualUnitBlue!72,line width=1.05pt,dashed},
  same marginal/.style={draw=ualNeutral,line width=1.2pt}
]
  \path[use as bounding box] (0,0) rectangle (160,56);

  \node[ual panel title] at (3,53.2)
    {(a) Heterogeneous fixed-unit laws};
  \pic at (14,35) {ual plot axis};
  \node[rotate=90,anchor=south,font=\scriptsize,text=ualNeutral]
    at (10.4,43) {$P(Y=1\mid x,u)$};
  \draw[unit A] (14,47) -- (42,47);
  \draw[unit B] (14,39) -- (42,39);
  \draw[same marginal] (14,43) -- (42,43);
  \node[anchor=west,font=\scriptsize,text=ualUnitBlue]
    at (42.6,47.0) {$u_A$};
  \node[anchor=west,font=\scriptsize,text=ualUnitBlue!72]
    at (42.6,39.0) {$u_B$};
  \node[ual tiny note,text width=38mm,align=left] at (66,44.2)
    {$p_{u_A}=\tfrac34$, $p_{u_B}=\tfrac14$,\\equal weights $\tfrac12$, $\tfrac12$};

  \node[ual panel title] at (3,26.6)
    {(b) unit-independent law};
  \pic at (14,8) {ual plot axis};
  \node[rotate=90,anchor=south,font=\scriptsize,text=ualNeutral]
    at (10.4,16) {$P(Y=1\mid x,u)$};
  \draw[same marginal] (14,16) -- (42,16);
  \node[ual tiny note,text width=38mm,align=left] at (66,16.5)
    {one law for every $u$,\\$P(Y=1\mid x,u)=\tfrac12$};

  \draw[ual panel divider] (2,29.2) -- (88,29.2);
  \draw[ual panel divider] (114,4) -- (114,53);

  \node[ual pill] (mix) at (101,29) {mix over $U$};
  \draw[ual flow] (88,43) -- (mix.west);
  \draw[ual flow] (88,16) -- (mix.west);
  \draw[ual flow] (mix.east) -- (114,29);

  \node[ual panel title] at (117,53.2) {Observable marginal};
  \pic at (124,20) {ual plot axis};
  \node[rotate=90,anchor=south,font=\scriptsize,text=ualNeutral]
    at (120.4,28) {$P(Y=1\mid x)$};
  \draw[same marginal] (124,28) -- (152,28);
  \node[anchor=south,font=\scriptsize] at (148.5,28.4) {$\tfrac12$};
  \node[ual note,text width=36mm] at (138,12.0)
    {same marginal law};
\end{tikzpicture}

%% file: figures/fig4-repeated-units.tex
\begin{tikzpicture}[font=\footnotesize,x=1mm,y=1mm]
  \path[use as bounding box] (0,0) rectangle (160,58);

  \node[ual panel title] at (3,55.4)
    {(a) Repeated units change the estimand};

  \node[anchor=west,font=\footnotesize\bfseries,text=ualNeutral]
    at (4,50.0) {Row average};
  \node[anchor=east,ual tiny note] at (74,50.0)
    {weight $m(u)/N$};

  \foreach \x/\lab in {10/{A_1},19/{A_2},28/{A_3}}
    \node[ual chip,ual identity A] at (\x,43.6) {$\lab$};
  \node[ual chip,ual identity B] at (40,43.6) {$B_1$};
  \foreach \x/\lab in {52/{C_1},61/{C_2}}
    \node[ual chip,ual identity C] at (\x,43.6) {$\lab$};

  \draw[ual brace,draw=ualIdentityA] (5.5,38.6) -- (32.5,38.6);
  \node[ual tiny note,text=ualIdentityA] at (19,34.6) {$3/N$};
  \draw[draw=ualIdentityB,dashed,line width=0.9pt]
    (36.4,38.8) -- (43.6,38.8);
  \node[ual tiny note,text=ualIdentityB] at (40,34.6) {$1/N$};
  \draw[ual brace,draw=ualIdentityC,dash dot] (47.5,38.6) -- (65.5,38.6);
  \node[ual tiny note,text=ualIdentityC] at (56.5,34.6) {$2/N$};

  \draw[densely dotted,line width=0.65pt,draw=ualNeutral!55]
    (4,31.2) -- (74,31.2);
  \node[anchor=west,font=\footnotesize\bfseries,text=ualNeutral]
    at (4,27.4) {Per-unit average};
  \node[anchor=east,ual tiny note] at (74,27.4)
    {weight $1/M$};

  \node[ual unit,ual identity A,minimum size=9.5mm] at (15,16.8) {$u_A$};
  \node[ual unit,ual identity B,minimum size=9.5mm] at (39,16.8) {$u_B$};
  \node[ual unit,ual identity C,minimum size=9.5mm] at (63,16.8) {$u_C$};
  \node[ual tiny note,text=ualIdentityA] at (15,8.2) {$\bar\ell(u_A)$};
  \node[ual tiny note,text=ualIdentityB] at (39,8.2) {$\bar\ell(u_B)$};
  \node[ual tiny note,text=ualIdentityC] at (63,8.2) {$\bar\ell(u_C)$};

  \draw[ual panel divider] (80,2.0) -- (80,56.4);
  \node[ual panel title] at (84,55.4)
    {(b) Repeated units change the split question};

  \node[anchor=west,font=\footnotesize\bfseries,text=ualNeutral]
    at (84,50.4) {Record-wise split};

  \node[ual split train,minimum width=34mm,minimum height=16mm]
    (train-rw) at (102.5,39.6) {};
  \node[ual split test,minimum width=34mm,minimum height=16mm]
    (test-rw) at (139.5,39.6) {};
  \node[ual tiny note] at (102.5,45.4) {train};
  \node[ual tiny note] at (139.5,45.4) {test};
  \foreach \x/\lab in {91.5/{A_1},102.5/{A_2}}
    \node[ual chip,ual identity A] at (\x,37.6) {$\lab$};
  \node[ual chip,ual identity B] at (113.5,37.6) {$B_1$};
  \node[ual chip,ual identity A] at (131.5,37.6) {$A_3$};
  \node[ual chip,ual identity B] at (147.5,37.6) {$B_2$};
  \node[ual tiny note] at (121,29.4)
    {known unit, new event};

  \draw[densely dotted,line width=0.65pt,draw=ualNeutral!55]
    (84,26.8) -- (156,26.8);
  \node[anchor=west,font=\footnotesize\bfseries,text=ualNeutral]
    at (84,23.6) {Whole-unit split};

  \node[ual split train,minimum width=34mm,minimum height=16mm]
    (train-wu) at (102.5,13.6) {};
  \node[ual split test,minimum width=34mm,minimum height=16mm]
    (test-wu) at (139.5,13.6) {};
  \node[ual tiny note] at (102.5,19.4) {train};
  \node[ual tiny note] at (139.5,19.4) {test};
  \foreach \x/\lab in {91.5/{A_1},102.5/{A_2}}
    \node[ual chip,ual identity A] at (\x,11.6) {$\lab$};
  \node[ual chip,ual identity B] at (113.5,11.6) {$B_1$};
  \foreach \x/\lab in {131.5/{C_1},147.5/{C_2}}
    \node[ual chip,ual identity C] at (\x,11.6) {$\lab$};
  \node[ual tiny note] at (121,3.6)
    {new unit};
\end{tikzpicture}

%% file: appendix.tex
\section{Notation and Semantic Roles}
\label{app:notation}

The main text separates three layers that this appendix must not collapse.
The \emph{semantic} layer is the task-declared unit $u\in\mathcal U$ and its
response law $P_u^\star$. The \emph{exact mixture} layer is the first line of
\cref{eq:world-learner-response-coordinates}: unresolved identity mixes the
world targets $P_u^\star$ over units, without introducing learned parameters.
The \emph{computational} layer contains learner compositions under
\cref{hyp:shared-form-token}: the default token-space composition uses
$T_\phi(\,\cdot\mid\mathcal O,c)$ and $R_\theta$, while the theorem
specialization uses $Q_\phi$ and an identity-indexed response model. A learner
composition coincides with the exact world mixture only under the stated access,
response-realization, response-sufficiency, and external-query conditions.
World-side identity posteriors
$P(U\in du\mid\mathcal O)$ remain well-defined probability objects whether or
not the learner represents them. The theorems below use the identity-mixture
learner specialization; they are not a definition of unit abduction and do
not require the computational token to be an identity posterior. In
\cref{sec:consequences} the complete response input is written $W$; the
same object is the lowercase index $w$ of the kernel $K_u^w$ in
\cref{app:heterogeneity}.

\begin{table}[t]
\centering
\caption{Notation by semantic layer. Learner quantities are listed with the
computational pair first; $Q_\phi$ belongs to the identity-mixture learner
specialization.}
\label{tab:notation}
\footnotesize
\renewcommand{\arraystretch}{0.90}
\setlength{\tabcolsep}{3pt}
\begin{tabular}{@{}>{\raggedright\arraybackslash}p{0.18\linewidth}>{\raggedright\arraybackslash}p{0.24\linewidth}>{\raggedright\arraybackslash}p{0.51\linewidth}@{}}
\toprule
Symbol & Intended type & Boundary \\
\midrule
\multicolumn{3}{@{}l}{\textbf{World and population}} \\
$U:\Omega\to\mathcal U$ & population unit selection variable & which individual in $\mathcal U$ \\
$\Pi(du)$ & population unit law & which-individual randomness under $P$ \\
\addlinespace[2pt]
\multicolumn{3}{@{}l}{\textbf{Fixed individual and event}} \\
$u\in\mathcal U$ & realized persistent unit & fixed value of $U$; indexes a unit-conditioned law but is not synonymous with a learned numerical representation \\
$K_u(dx,dy)$ & fixed-individual law & retains exogenous/event variation after $U=u$ \\
$X,Y$ & population variables & observed values may supply pre-answer evidence; an observed evidence response is not the current target \\
$X_u,Y_u$ & fixed-individual variables & conditional laws under $U=u$ \\
$Y_u(a)$ & potential response & conventional causal notation for individual $u$ under treatment $a$ \\
\addlinespace[2pt]
\multicolumn{3}{@{}l}{\textbf{Observed record and attribution}} \\
$i$ & sample index & counts observations, not units \\
$(x_i,y_i)$ & realized record & supplies values that specify population events \\
$\mathcal D_U^{\rm world}$ & conceptual complete-data table & includes the definite $u_i$ whether or not attribution is observed \\
$\mathcal D_U^{\rm obs}$ & learner-visible-attribution table & special regime in which persistent-unit labels, or a trusted resolver that determines them, are learner-visible \\
$K=k$ and $u(k)$ & trusted key and resolved unit & stable referential and update address; determines which persistent unit a record concerns \\
\addlinespace[2pt]
\multicolumn{3}{@{}l}{\textbf{Learner quantities}} \\
$T_\phi(\mathcal O_u,c)$ & tokenizer & forms a contextual unit token from unit-attached factual evidence and context; may be deterministic or a kernel on token space; jointly learned with $R_\theta$ \\
$Z_u^c\in\mathcal Z$ & contextual unit token & learner-side representation supplied to $R_\theta$ in context $c$; not the unit \\
$R_\theta(dy\mid x,c,z)$ & shared response-law form & the same learned map for every unit; learner-side unit differences enter only through the token $z$ \\
$Q_\phi(du\mid\mathcal O)$ & which-unit belief & a law on $\mathcal U$ that may approximate $P(U\in du\mid\mathcal O)$ and feeds the identity-mixture learner; not the unit, not the token, and not the definition of unit abduction \\
$z_\theta(k)$ & ID-indexed lookup token & learned table row implementing \(Z_{u(k)}^c\) under direct access (embedding, factor, random effect, \ldots); not the unit \(u(k)\) \\
$f(x;u)$ & unit-conditioned predictor & fixed $u$ conditions a member of the predictor family \\
$f_0(x)$ & ordinary row-level predictor & collapse notation, not an average over units \\
$P_\theta(dy\mid x^Q,c^Q,u)$ & identity-indexed response model & aimed at $P_u^\star$ and used by the theorem's identity-mixture learner; under direct access this is $R_\theta(dy\mid x^Q,c^Q,z_\theta(k))$ \\
\addlinespace[2pt]
\multicolumn{3}{@{}l}{\textbf{Evidence, query, and target}} \\
$\mathcal O$ & factual evidence & observed-event information available before the answer and declared usable for unit inference \\
$P(U\in du\mid\mathcal O)$ & evidence-conditioned unit posterior & conditional law of the already-realized unit under the data-generating distribution \\
$x^Q$ & response query & supplied response argument; not new abductive evidence merely by being queried \\
$c^Q$ & explicit response context & answer-time information declared to change $Y^Q$ directly beyond $(U,x^Q)$; omitted when empty \\
$W$ & complete response input in \cref{sec:consequences} & $W=X^Q$ or $W=(X^Q,C^Q)$; not the unit token $Z_u^c$ \\
$K_u^w$ & unit-conditioned response kernel at query $w$ & $w=(x^Q,c^Q)$ is the realized complete response input; not a token \\
$Y^Q$ & current response target & is excluded from the pre-answer evidence \\
$P_u^\star(dy\mid x^Q,c^Q)$ & population response conditional & shorthand for $P^\star(Y^Q\in dy\mid X^Q=x^Q,C^Q=c^Q,U=u)$ under response sufficiency \\
\bottomrule
\end{tabular}
\end{table}

As summarized in \cref{tab:notation}, the domain $\mathcal U$ belongs to the
declared population learning problem. Dataset indices label records rather than
creating new unit domains. A value $x_i$ instead specifies
the event $\{X=x_i\}$; a pair $(x_i,y_i)$ may specify
$\{X=x_i,Y=y_i\}$. If a study makes attribution learner-visible, lowercase
$u_i$ may annotate its records, and $u_i=u_j$ then records a supplied
same-individual relation. This metadata provides direct unit access; it does not
make a model-side representation observed. Unit abduction applies when no
resolver identifies the unit. Under direct access,
lowercase $x_i,y_i$
are event-level realized values associated with the fixed-unit event law
$K_{u_i}(dx,dy)$. The notation keeps a record value distinct from the fixed-unit
stochastic object, so we avoid writing $X_i:=X_{u_i}$ or $Y_i:=Y_{u_i}$.
Repeated observations require a separately specified joint law; the population
unit variable remains $U$ rather than a foundational family $(U_i)$.

Let $\Pi(du)$ denote the population law of $U$, and let $K_{X,u}(dx)$ be the
fixed-individual observation kernel.  Then
\begin{equation}
  P_X(dx)=\int_{\mathcal U}K_{X,u}(dx)\,\Pi(du),
  \qquad
  P(X_U\in A\mid U=u)=K_{X,u}(A).
  \label{eq:appendix-selected-variable}
\end{equation}
The first variation comes from drawing $U$; the second remains inside
$K_{X,u}$ after an individual is fixed. A realized record contributes an event
from the selected unit rather than a new unit variable $U_i$. Repeated-record
sampling and uncertain linkage require an explicitly declared joint observation
model.

\paragraph{Evidence is event-level information.}
Uppercase symbols denote random variables; lowercase symbols denote realized
values.  A positive-probability event $E$ permits ordinary conditioning
$P(U\in du\mid E)$. When the protocol records evidence at a realized value
$o$, $P(U\in du\mid\mathcal O=o)$ denotes evaluation of a regular conditional
distribution. We use $\mathcal O$ as protocol-level shorthand for either form
of pre-answer information.  Thus $\{X=x_i,Y=y_i\}$ is shorthand
for observed evidence values, whereas $(x_i,y_i)$ is their realized-value
encoding.  Any $y_i$ admitted here is a historical or other already observed
response, not the current target response. A designed candidate $x^Q$ supplied
to $P_\theta(dy\mid x^Q,c^Q,u)$ is a response query. It contributes evidence
about $U$ only when it was factually observed before the answer and the protocol
incorporates it into $\mathcal O$.

When no resolver identifies the realized unit, unit abduction forms a
contextual unit token as in \cref{eq:unit-abduction}; the default prediction is
the token composition \cref{eq:token-composition}. The identity-mixture
specialization used by the theorems is \cref{eq:core-composition}. These
displays are not repeated here. Direct lookup and abduction differ by
provenance, not by the type of the object they return.

In the basic supervised protocol,
$\mathcal O_i=\{X=x_i\}$ determines the world-side conditional
$P(U\in du\mid X=x_i)$. The target $y_i$ scores the learner mixture but is
excluded from $\mathcal O_i$. A history encoder, probabilistic identity
rule, or other map may implement $T_\phi$; the primitive is
architecture-independent.

The response-sufficiency condition is
\begin{equation}
  Y^Q\perp\!\!\!\perp\mathcal O\mid(U,x^Q,c^Q),
  \label{eq:appendix-response-sufficiency}
\end{equation}
with $c^Q$ omitted when $(U,x^Q)$ suffices. This assumption applies when
information in the factual record that directly changes the response beyond its
role in forming unit information is represented explicitly by $c^Q$, rather
than by an undeclared dependence of the response kernel on the full evidence
bundle. The same raw observation may support both roles when the protocol
declares them separately. On the learner side, the role of $\mathcal O$ in
the response is mediated by the formed token after that context has been
represented. To identify the exact response conditional with a mixture using
the fixed identity law $P(U\in du\mid\mathcal O)$, the protocol additionally
satisfies the external-query condition in \cref{eq:external-query-contract}.
Otherwise the exact mixing law is generally
$P(U\in du\mid\mathcal O,x^Q,c^Q)$. This is not permission to update
formed unit information from a merely proposed query: any unit-informative
observed variable must be declared as factual evidence before the token, or,
in the identity specialization, the belief, is formed.

\paragraph{Observed identifiers and direct unit access.}
If an ID event uniquely determines which persistent unit is present, a trusted
resolver implements \(k\mapsto u(k)\) and fixes a stable address. A separate
lookup at that address supplies the token, typically the learned row
\(Z_{u(k)}^c=z_\theta(k)\). The resolver fixes the referent; the lookup selects
the token slot to read or update. Neither operation makes \(z_\theta(k)\) into the unit.
Learning \(z_\theta(k)\) is fitting the token, not a which-unit inference. The
corresponding which-unit conditional $P(U\in du\mid\mathrm{ID}=k)$ is a point
mass (a Dirac measure). The identity-mixture specialization then uses
$\delta_{u(k)}$; no learned $Q_\phi$ module need be run, and the Dirac mass is
not an executed abduction. Crucially, the Dirac mass is on the stable referential unit $u(k)$,
not on the trainable and potentially reparameterized token $z_\theta(k)$. A
history-dependent state for the resolved unit belongs to the tokenizer's
context, or to the response layer, and may change without changing this
identity law.

\section{Conceptual Clarifications and Inferential Distinctions}
\label{app:interface-distinctions}

This section is the boundary register for the main text. Four layers should be
kept separate: the task-side referent $u$, the learner-side token $Z_u^c$, the
world-side response family $\{P_u^\star\}$, and the deployed prediction produced
by a particular learner formulation. A trusted key resolves a referent and a
lookup selects a token; unit abduction forms a token from factual evidence.
Neither access procedure fixes the token parameterization or supplies a
similarity rule between units. Likewise, a unit-omitting learner can fit a
row-level marginal in either a homogeneous or heterogeneous world, while the
identity-mixture specialization $Q_\phi$ is a theorem-specific formulation
rather than the definition of the tokenizer. The subsections below collect the
associated semantic, evidence, and identification qualifications so that the
main text can state the positive model without repeating every boundary case.

\subsection{Persistent units and model representations}

The declared unit domain $\mathcal U$ carries the task's persistent-referent
semantics. A learned embedding, preference factor, or random effect may
parameterize its lookup token, but is not thereby the realized value of $U$. In
a trusted-ID model, $u(k)$ is the resolved unit and $z_\theta(k)$ is
the token $Z_{u(k)}^c$. The token may be transformed jointly with the downstream
model without changing predictions, so its coordinates need not be unique or
scientifically identifiable \citep{bengio2013representation}. A posterior over
such response-side coordinates is likewise not automatically the which-unit
belief $Q_\phi(du\mid\mathcal O)$, and neither object is a sample-inclusion
probability or an importance weight.

\subsection{Tokens, which-unit beliefs, and point summaries}

The default learner-side object is the contextual unit token $Z^c$.
In the identity-belief specialization, $Q_\phi(du\mid\mathcal O)$ instead
targets the world conditional $P(U\in du\mid\mathcal O)$ when the protocol
declares that objective; calibration and recovery require additional
conditions. Means and selected modes, when defined, are deterministic
functionals of this belief, whereas a computational draw satisfies
$\widetilde U_\phi\mid\mathcal O\sim
Q_\phi(\cdot\mid\mathcal O)$. The response-level prediction propagates the
full belief through \cref{eq:core-composition}; a downstream decision
may then extract a point prediction, interval, action, or abstention.

\subsection{Record, evidence, and query roles}

The notation table separates record values, factual evidence, response queries,
and current targets. In particular, $i$ indexes an event rather than defining a
world variable $U_i$, and a supplied candidate $x^Q$ is evaluated at the token
formed from the fixed evidence cutoff. A newly observed factual event may update
that token, and may update $Q_\phi$ in the identity-belief specialization; the
current target remains excluded from pre-answer evidence.

If a response-relevant context is neither fixed by the protocol nor absorbed
into $x$, omitting it from the response specification generally yields, for compatible
regular-conditional versions and $P_{X,U}$-almost every $(x,u)$,
\begin{equation}
  P_u^\star(dy\mid x)
  =\int P_u^\star(dy\mid x,c)\,
    P^\star(dc\mid x,U=u).
  \label{eq:context-marginalization}
\end{equation}
Consequently, unit-dependent context distributions can create
heterogeneity at the marginal response level even when the context-conditional response
surface is common across units. Conversely, context-specific differences
between units can cancel after integration and disappear from the marginal
response specification. Declaring $c^Q$ is therefore optional at the level of mathematical
typing but substantive whenever the task must distinguish which individual
differs from which condition that individual occupies.

An evidence bundle may contain several observed values without asserting that
they concern the same individual. Repeated-individual attribution and entity
linkage require an additional joint observation model.

\subsection{Marginalization, erasure, and component identification}

Three operations associated with unit erasure have different targets. At the
prediction level, marginalizing the exact world law over the conditional unit
distribution yields the row-wise law in
\cref{prop:single-row-collapse}. At the response-structure level,
$f(x;u)=f_0(x)$ for every relevant $u$ is a sufficient invariance condition for
a sample-only rule $f_0$; it does not define $f_0$ as an average of
$f(x;u)$. At the data level, deleting attribution columns from a table is a
deterministic map, not a mixture. The first is probabilistic marginalization,
the second is a response restriction, and the third is a data edit.

For the identity-belief specialization of the composition, the exact world
mixture equals $P(Y\in dy\mid X=x_i)$ by
\cref{prop:single-row-collapse}. When a density or probability mass exists and
$\mathcal O_i=\{X=x_i\}$, the learner score in that specialization is
\begin{equation}
  \log\widehat p_{\theta,\phi}(y_i\mid\mathcal O_i;\,x_i)
  =\log\!\int_{\mathcal U}
    p_\theta(y_i\mid x_i,u)Q_\phi(du\mid\mathcal O_i).
  \label{eq:appendix-rowwise-log-score}
\end{equation}
The default token-space prediction instead scores
$\int R_\theta(dy_i\mid x_i,c,z)\,T_\phi(dz\mid\mathcal O_i,c)$. Here the
observed $x_i$ serves both as pre-target evidence and as the response
covariate, while an alternative $x^Q$ reuses the same formed token. The log of
the mixture is generally different from the mixture of conditional log scores;
for continuous $Y$, the likelihood is a density rather than a point
probability. Matching the marginal law alone does not identify $T_\phi$ or
$Q_\phi$, recover a persistent unit, or establish response dependence on $u$.

\section{Unit-Conditioned Response Heterogeneity}
\label{app:heterogeneity}

Throughout this section, $w=(x^Q,c^Q)$ denotes a declared complete response
input---not a unit token $Z_u^c$. This is the same object written $W$ in
\cref{sec:consequences}. The unit primitive makes it possible to
define response heterogeneity without committing to a particular random-effect
parameter, neural architecture, or clustering device. This section isolates
that definition, separates it from the structural assumptions needed for
learning, and then locates several established literatures in the resulting
coordinates. Its scope is the precise object of heterogeneity, the
restrictions that make it learnable, and its relation to neighboring forms of
variation. \Cref{hyp:shared-form-token} is one such restriction: differences
among $K_u^w$ must be expressed by contextual unit tokens under a shared form
$R_\theta$.

\subsection{The response-law profile}

Let $w=(x^Q,c^Q)$ denote a declared response query, with $c^Q$ omitted when it
is empty, and let $\mathfrak Q$ be a common admissible query family for the
task. The task-level response object is a declared jointly measurable Markov
kernel $(u,w)\mapsto K_u^w$ on $\mathcal U\times\mathfrak Q$. Write
\begin{equation}
  K_u^w(dy)
  :=P_u^\star(dy\mid x^Q,c^Q).
  \label{eq:heterogeneity-response-profile}
\end{equation}
For fixed $u$, the family $w\mapsto K_u^w$ gives the response law over the
declared query family. Equality of two such families means equality of the
declared kernels on $\mathfrak Q$. If the kernel is learned or interpreted only
through a regular conditional law under a design $\Lambda(du,dw)$, its
off-support values are modeling choices rather than identified features. The
$\Lambda$-relative notion below states the corresponding observable condition.

\begin{definition}[Kernel formulation of unit-response homogeneity]
\label{def:unit-response-heterogeneity}
For $U\sim\Pi$, a declared query family $\mathfrak Q$, and the chosen jointly
measurable kernel in \cref{eq:heterogeneity-response-profile}, the task is
\emph{unit-response homogeneous} when there are a single $\Pi$-null set
$N\subseteq\mathcal U$ and a common response kernel $w\mapsto K^w$ such that
$K_u^w=K^w$ for every $u\notin N$ and every $w\in\mathfrak Q$. It is
\emph{unit-response heterogeneous} otherwise.
\end{definition}

Equivalently, define query-relative response equivalence by
\begin{equation}
  u\equiv_{\mathfrak Q}v
  \quad\Longleftrightarrow\quad
  K_u^w=K_v^w
  \ \text{as probability laws for every declared }w\in\mathfrak Q.
  \label{eq:query-response-equivalence}
\end{equation}
Homogeneity says that $\Pi$ is concentrated on one such response-equivalence
class; heterogeneity says that it is not. This quotient is predictive, not
referential. Distinct individuals may satisfy
$u\equiv_{\mathfrak Q}v$ while remaining distinct units. Conversely,
changing $P(X\mid U=u)$ or the frequency with which a unit is sampled does not
by itself establish heterogeneity in $K_u^w$.

For a declared design $\Lambda$ on $\mathcal U\times\mathfrak Q$, call the
family \emph{$\Lambda$-visible homogeneous} when there is a kernel $K^w$ such
that $K_u^w=K^w$ for $\Lambda$-almost every $(u,w)$, and
\emph{$\Lambda$-visible heterogeneous} otherwise. Structural homogeneity in
Definition~\ref{def:unit-response-heterogeneity} implies visible homogeneity
for every compatible design, where compatibility means that the unit marginal
$\Lambda_U$ is absolutely continuous with respect to $\Pi$. The converse can
fail when a design omits queries or units on which response laws differ. This
qualification also makes the notion invariant to changes of a regular
conditional law on a $\Lambda$-null set.

The definition is query-relative for a substantive reason. Two units may be
equivalent on the queries in one study and distinguishable on a larger query
family. A treatment, time point, item, or context can be part of $w$ without
becoming a new unit. The response kernels in
\cref{eq:heterogeneity-response-profile} are observational unless a separate
causal assignment and identification assumptions license an interventional
reading.

\subsection{Response-law model classes}

The family $K_u^w$ supports several qualitatively different model classes. The
restrictions in \cref{tab:heterogeneity-regimes} distinguish their response
objects and transfer assumptions.

\FloatBarrier
\begin{table}[htbp]
\centering
\caption{Response-law cases and model classes inside a unit-explicit task.
Standard sample-only learning is listed separately because its basic object
does not require a declared unit population.}
\label{tab:heterogeneity-regimes}
\footnotesize
\begin{tabular}{@{}>{\raggedright\arraybackslash}p{0.22\linewidth}
                    >{\raggedright\arraybackslash}p{0.30\linewidth}
                    >{\raggedright\arraybackslash}p{0.40\linewidth}@{}}
\toprule
Case or class & Restriction & Meaning \\
\midrule
sample-only & no persistent unit is declared & Learns a row-level law such as
$P(dy\mid x)$; it makes no direct statement about persistent individuals. \\
homogeneous unit extension & $K_u^w=K^w$ for $\Pi$-almost every $u$ & After persistent
units are declared, unit identity adds no response information at the
queries in $\mathfrak Q$. \\
pure-unit, input-invariant & $K_u^w=K_u$ while $K_u$ varies with $u$ & Outcomes
vary across units while remaining invariant to the supplied query. \\
structured unit-responsive class & $K_u^w=K^w(\eta,\zeta_u)$; simple
relation in a token, linear running instance
(\cref{def:linear-shared-form}) & Both query and unit may matter. The
running instance is a linear predictor
$\alpha_\theta+\langle\psi_\theta,Z\rangle$ in the unit token $Z$;
other declared simple maps of $Z$ instantiate the same class.
$\psi_\theta$ may be nonlinear in the query $w$. \\
saturated unit-responsive class & $(u,w)\mapsto K_u^w$ is otherwise
unrestricted across units & The class permits one arbitrary response family per
unit, including homogeneous families; unseen-unit learning requires additional
structure or information. \\
\bottomrule
\end{tabular}
\end{table}

The final two rows describe model classes that permit heterogeneity, not a
claim that every member is heterogeneous. Whether the realized response family
is homogeneous or heterogeneous is still decided by
Definition~\ref{def:unit-response-heterogeneity}, or by its
$\Lambda$-visible counterpart for a deployment-limited claim.

Standard supervised learning does not, by its basic sample-indexed notation
\citep{bishop2006pattern,hastie2009elements,
shalevshwartz2014understanding}, specify a population of persistent units over
which homogeneity could be assessed. A homogeneous
specialization results after those units are declared and the second row of
\cref{tab:heterogeneity-regimes} is imposed. The exact marginal collapse in
\cref{prop:single-row-collapse} is a different statement: heterogeneous
fixed-unit laws may integrate to the same row-level predictive law.

\subsection{Shared Structure and Learnability}

Definition~\ref{def:unit-response-heterogeneity} determines whether the response
family varies with $u$ but does not specify how the families across units are
related. Statistical learnability therefore depends on a declared structural
restriction of the form
\begin{equation}
  (u,w)\longmapsto K_u^w
  \quad\text{belongs to a declared structured class of kernels},
  \label{eq:shared-structure-class}
\end{equation}
which may encode shared parameters, a hierarchy, low-rank or smooth variation,
a shared representation with unit-specific heads, a finite set of response
types, or another coupling across units. These are alternative structural
assumptions rather than additions to the definition of heterogeneity.

A saturated class of measurable unit-to-law maps permits heterogeneity but does
not by itself support transfer from observed units to an unseen unit. A
learnability analysis specifies which response structure is shared, which part
varies, what evidence constrains the varying part, and the deployment population
and query family under which that structure is testable. Under
\cref{hyp:shared-form-token} the
shared restriction is the pair $(T_\phi,R_\theta)$, with a simple relation in
the token as the structured class and
\cref{def:linear-shared-form} as the running instance: $R_\theta$ depends on $z$
only through a declared low-complexity map, linearly in finite dimension or
through another declared simple map of $z$. Random effects and
multitask representations provide two established forms of such structure
\citep{gelman2006data,caruana1997multitask,baxter2000model}; a full learnability
analysis must additionally specify the loss, sampling regime, model complexity,
and learning rule \citep{shalevshwartz2014understanding}.

\subsection{Oracle Value Under Proper Scoring Rules}

The definition also has a decision-theoretic reading. Let
$\Lambda(du,dw)$ be a declared deployment design over units and queries, let
$Y\mid(U=u,W=w)\sim K_u^w$, and suppose a regular conditional
$\Lambda(du\mid w)$ has been fixed. Define the pooled Bayes law
\begin{equation}
  \overline K^w(dy)
  :=\int_{\mathcal U}K_u^w(dy)\,\Lambda(du\mid w).
  \label{eq:pooled-bayes-law}
\end{equation}
Let the admissible reports range over a class $\mathcal P$ of probability laws
that contains $K_u^w$ for $\Lambda$-almost every $(u,w)$ and is closed under
the displayed mixtures, so that $\overline K^w\in\mathcal P$. Let $S(P,y)$ be
a strictly proper predictive loss on $\mathcal P$, with smaller values better,
and let
$D_S(P,Q):=E_{Y\sim P}[S(Q,Y)-S(P,Y)]$ be its associated nonnegative regret
divergence \citep{gneiting2007strictly}.

\begin{proposition}[Oracle heterogeneity gap under a proper score]
\label{prop:proper-score-heterogeneity-gap}
Assume the displayed expectations are finite. The Bayes-risk difference
between the best predictor that observes $(U,W)$ and the best predictor that
observes only $W$ is
\begin{align}
  R_{\rm pool}-R_{\rm Unit}
  &:=E\!\left[S(\overline K^W,Y)-S(K_U^W,Y)\right] \\
  &=E_{(U,W)\sim\Lambda}
    \left[D_S\!\left(K_U^W,\overline K^W\right)\right]
  \geq 0.
  \label{eq:proper-score-heterogeneity-gap}
\end{align}
Under strict propriety, equality holds if and only if
$K_U^W=\overline K^W$ for $\Lambda$-almost every $(U,W)$. For logarithmic
loss, the gap is $I_\Lambda(U;Y\mid W)$ whenever the conditional mutual
information is well defined.
\end{proposition}

\begin{proof}
Condition on $(U,W)=(u,w)$. By the definition of the scoring-rule regret, the
conditional expected excess loss from reporting $\overline K^w$ instead of the
true $K_u^w$ is $D_S(K_u^w,\overline K^w)$. Integrating with respect to
$\Lambda$ proves the identity and nonnegativity. Strict propriety makes the
regret zero exactly when the two conditional laws agree almost surely. Under
logarithmic loss the regret is the Kullback--Leibler divergence, whose
conditional expectation is $I_\Lambda(U;Y\mid W)$.
\end{proof}

This proposition applies the classical proper-scoring-rule identity to give an
exact, query- and deployment-relative magnitude to $\Lambda$-visible response
heterogeneity.
It is an oracle comparison and does not establish that $U$ is observed,
identifiable, causally useful, or learnable from finite selectively observed
events.

\FloatBarrier

\section{DiscoSCM as a Causal Specialization of the Unit Primitive}
\label{app:discoscm-specialization}

DiscoSCM, introduced by \citet{gong2024discoscm}, provides a substantive causal
specialization of the unit primitive. This relation is stronger than a shared
use of the letter $U$: DiscoSCM separates the selected individual from
event-level exogenous variation, conditions structural causal mechanisms on the
selected unit, and derives intervention and counterfactual quantities within
that unit-conditioned model. At the same time, the causal and cross-world
assumptions of DiscoSCM are additional structure; they do not follow from the unit primitive
alone.

\subsection{From a Response-Law Family to Unit-Conditioned Structural Mechanisms}

The generic supervised specialization in this paper begins with a population
law $U\sim\Pi$ and a unit-conditioned response family
$u\mapsto K_u^w(dy)$. DiscoSCM refines the response family into a structural
causal model
\begin{equation}
  \mathcal M^{\rm D}
  =\langle U,\mathbf E,\mathbf V,\mathcal F\rangle,
  \qquad
  V_j\leftarrow f_j(\operatorname{pa}_j,E_j;U),
  \label{eq:discoscm-unit-structural-model}
\end{equation}
where $U=u$ selects the individual, $\mathbf E$ contains exogenous variables,
$\mathbf V$ contains endogenous variables, and $\mathcal F$ is a family of
unit-conditioned structural assignments. The DiscoSCM formulation assumes
$U\perp\!\!\!\perp\mathbf E$ in its basic construction. This assumption is
specific to that construction; the general unit primitive does not require
independence between unit selection and every source of event variation.

For a treatment or supplied cause $X=x$ and outcome $Y$, suppressing other
parents for clarity, the structural assignment induces the fixed-unit response
kernel
\begin{equation}
  K_{u,\mathrm{str}}^{x}(B)
  :=P\!\left(f_Y(x,E_Y;u)\in B\right).
  \label{eq:discoscm-induced-response-kernel}
\end{equation}
Thus $u\mapsto K_{u,\mathrm{str}}^x$ is a structurally induced
unit-conditioned response family, while the mechanisms
$u\mapsto f_Y(\cdot,\cdot;u)$ provide its richer causal representation. It is
not identified with the observational conditional $P_u^\star(dy\mid x)$
without additional conditions linking observational and interventional laws.
A unit can alter the baseline response, the
effect of treatment, the outcome noise law, or a larger causal mechanism. The
shared causal graph and the form of $\mathcal F$ then provide one possible
answer to the shared-structure question posed in
\cref{sec:response-law-constraints}.

The separation between $U$ and $\mathbf E$ is the central correspondence. A
realized event is generated by both an individual value $u$ and an exogenous
realization $\mathbf e$. Holding $u$ fixed preserves the referential individual;
it does not freeze the event realization. In the notation of this paper,
DiscoSCM therefore instantiates the distinction between which-individual
variation and the stochastic variation retained by a fixed-unit kernel.

\begin{table}[htbp]
\centering
\caption{Object-level correspondence between the unit formalism and DiscoSCM.
The final column records structure contributed by the causal specialization.}
\label{tab:discoscm-unit-correspondence}
\footnotesize
\begin{tabular}{@{}>{\raggedright\arraybackslash}p{0.25\linewidth}
                    >{\raggedright\arraybackslash}p{0.27\linewidth}
                    >{\raggedright\arraybackslash}p{0.38\linewidth}@{}}
\toprule
Unit formalism & DiscoSCM object & Additional causal content \\
\midrule
population law $U\sim\Pi$ & unit selection variable $U$ with law $P(u)$ &
Population over individuals participating in a structural causal system. \\
realized unit $u$ & individual retained across factual and counterfactual worlds
& The same referent appears in every structural assignment
$f_j(\cdot,\cdot;u)$. \\
fixed-unit event variation & exogenous variables $\mathbf E$ & A realized event
depends on both $(u,\mathbf e)$; fixing $u$ need not make the outcome
deterministic. \\
response law $K_u^w$ & law induced by
$V_j\leftarrow f_j(\operatorname{pa}_j,E_j;u)$ & The response is generated by a
directed mechanism rather than specified only as a conditional kernel. \\
response query $x^Q$ & intervention value $x$ in $do(X=x)$ & Reading a query as
an intervention requires the causal graph, structural assignments, and an
intervention semantics. \\
factual evidence $\mathcal O$ & an observed factual trace & Evidence updates
unit information about the already-realized unit; DiscoSCM represents that
update as an identity posterior, which is one specialization of a formed
token. The trace may also contain variables with separately declared direct
response roles. \\
conditional unit law $P(du\mid\mathcal O)$ & DiscoSCM unit posterior
$P(du\mid\mathcal O)$ & DiscoSCM's population-level Layer-3 reduction mixes
against the exact identity posterior. That is DiscoSCM's abduction object, not
the definition of unit abduction in this paper. A learned $Q_\phi$ approximates
that identity law when the protocol uses the identity-belief specialization. \\
\bottomrule
\end{tabular}
\end{table}

\subsection{The additional causal and cross-world structure}

The unit primitive preserves the referent across a query family but is
observational by default. DiscoSCM adds an intervention operation and a family
of counterfactual exogenous variables $\mathbf E(x)$. For an outcome mechanism,
the resulting fixed-unit counterfactual outcome has the form
\begin{equation}
  Y_u^{d}(x)=f_Y(x,E_Y(x);u).
  \label{eq:discoscm-fixed-unit-counterfactual}
\end{equation}
The superscript $d$ distinguishes this distribution-consistent counterfactual
from a traditional same-noise SCM counterfactual. The defining marginal
restriction is
\begin{equation}
  \mathbf E(x)\ \overset{d}{=}\ \mathbf E,
  \label{eq:discoscm-noise-distribution-consistency}
\end{equation}
which yields, under the model's structural conditions, the fixed-unit
distribution-consistency relation
\begin{equation}
  \mathcal L\!\left(Y_u^{d}(x)\mid X=x,U=u\right)
  =\mathcal L\!\left(Y\mid X=x,U=u\right).
  \label{eq:discoscm-distribution-consistency}
\end{equation}
This equality in distribution replaces the pointwise equality imposed by the
usual consistency relation.

Equation~\eqref{eq:discoscm-noise-distribution-consistency} fixes the marginal
law in each world but does not by itself determine the coupling among
$\mathbf E$, $\mathbf E(x)$, and $\mathbf E(x')$. That coupling is a Layer-3
modeling choice because it determines joint cross-world quantities. Traditional
SCM semantics occupy the same-noise endpoint
$\mathbf E(x)=\mathbf E$ almost surely. The population-valuation result of
\citet{gong2024discoscm} instead uses the additional assumption
$\mathbf E(x)\perp\!\!\!\perp\mathbf E$ for the factual and queried worlds.
Other couplings can share the same Layer-1 and Layer-2 marginals while inducing
different Layer-3 joint distributions. The unit primitive identifies the
referent that is held fixed across these worlds; DiscoSCM specifies how the
remaining cross-world randomness is related.

\subsection{Unit Abduction, causal valuation, and reduction}

Let $\mathcal O$ denote an observed factual trace. The law of total
probability gives the exact decomposition
\begin{align}
  P\!\left(Y^{d}(x)\in dy\mid\mathcal O\right)
  =\int_{\mathcal U}
    &P\!\left(Y^{d}(x)\in dy
      \mid\mathcal O,U=u\right)
    P(du\mid\mathcal O).
  \label{eq:discoscm-exact-unit-decomposition}
\end{align}
Under the DiscoSCM condition that the queried counterfactual noise is
independent of the factual trace after fixing the unit, the first integrand no
longer depends on $\mathcal O$. Writing its fixed-unit counterfactual kernel as
$K_{u,\mathrm{cf}}^x$, the decomposition becomes
\begin{equation}
  P\!\left(Y^{d}(x)\in dy\mid\mathcal O\right)
  =\int_{\mathcal U}K_{u,\mathrm{cf}}^x(dy)
    P(du\mid\mathcal O).
  \label{eq:discoscm-abduction-valuation-reduction}
\end{equation}
This is the causal specialization of the identity-belief composition in
\cref{eq:core-composition}. DiscoSCM abducts a posterior over identity: factual
evidence updates $P(du\mid\mathcal O)$; the causal model evaluates the
counterfactual response law for each fixed $u$; and integration reduces the
unit-specific valuations to a population-level answer.
\citet{gong2024discoscm} call these stages \emph{abduction}, \emph{valuation},
and \emph{reduction}. That identity posterior is one admissible token; it is
not the definition of unit abduction in this paper. The default computational
prediction remains \cref{eq:token-composition}.

There are also two important differences from the generic learner formulation.
First, \cref{eq:discoscm-abduction-valuation-reduction} uses the exact world
identity conditional $P(du\mid\mathcal O)$, whereas this paper's default
abduction object is a formed token $Z^c$, and $Q_\phi(du\mid\mathcal O)$
is reserved for a learner identity belief that need not be calibrated.
Second, the kernel in DiscoSCM is interventional or counterfactual because the
structural model supplies that semantics; the kernel in the general supervised
formulation is observational unless additional causal assumptions are stated.

The stability result in \cref{prop:mixture-stability} applies to this causal
specialization after replacing $K_u$ by $K_{u,\mathrm{cf}}^x$, provided the
conditional independence used in
\cref{eq:discoscm-abduction-valuation-reduction} is stated directly and the
learned counterfactual kernel is measurable. It then separates counterfactual-
kernel approximation error from unit-belief error, with the latter scaled by
the diameter of the true fixed-unit counterfactual laws. This is a stability
statement for a specified causal target; it neither supplies the cross-world
coupling nor identifies the counterfactual kernel from observational data.

\subsection{Scope and Provenance of the Causal Specialization}

DiscoSCM is a \emph{causal theoretical specialization} of the unit primitive.
It declares the persistent unit, replaces a merely predictive response family
with unit-conditioned structural causal mechanisms, and states how factual
evidence and interventional or counterfactual queries interact with that
object. The hierarchy is
\begin{equation}
  \begin{aligned}
  \text{unit primitive}
  &\longrightarrow \text{unit-conditioned structural mechanisms}\\
  &\longrightarrow \text{DiscoSCM cross-world assumptions}\\
  &\longrightarrow \text{causal valuations and identification results}.
  \end{aligned}
  \label{eq:discoscm-specialization-hierarchy}
\end{equation}
The first arrow specializes the learned object; the second adds
distribution-consistency and a cross-world noise law; the third derives
DiscoSCM-specific causal consequences. This ordering also fixes the attribution
of results: the unit primitive supplies the persistent referent and the
population-to-individual decomposition, whereas DiscoSCM supplies the causal
mechanisms, intervention semantics, cross-world coupling assumptions, and the
theorems that depend on them.

Finally, unit abduction should not be conflated with the classical SCM
abduction of a full exogenous-noise realization. In this paper, unit abduction
forms a contextual unit token from factual evidence. That token may be an
identity posterior, but it need not be. DiscoSCM's abduction--valuation--
reduction step uses the identity posterior as its mixing measure, and
deliberately keeps that question separate from the coupling or resampling of
event-level noise across worlds. The two ``abduction'' operations are therefore
distinct even when DiscoSCM is read as a causal specialization of the unit
primitive.

\FloatBarrier

\section{Additional Formal Results}
\label{app:supporting-results}

The following results give technical statements and proofs supporting the
claims in the main text. They are consequences of retaining unit identity
rather than additional primitives. Unless a display is written in token
variables, the proofs use the identity-belief specialization
$Q_\phi(du\mid\mathcal O)$ of \cref{eq:core-composition};
that is the formulation in which value--access, mixture stability, and
single-row collapse are stated. It is not a claim that every tokenizer emits
a law on $\mathcal U$.

\subsection{Proof of the Value--Access Decomposition}
\label{app:value-access-proofs}

\begin{proof}[Proof of \cref{prop:value-access-decomposition}]
The conditional log-score identity gives
\begin{align}
  R_W^\star-R_{U,W}^\star&=I(U;Y^Q\mid W),\\
  R_W^\star-R_{\mathcal O,W}^\star
  &=I(\mathcal O;Y^Q\mid W).
\end{align}
The chain rule expands the same conditional mutual information in two ways:
\begin{align}
  I(U,\mathcal O;Y^Q\mid W)
  &=I(U;Y^Q\mid W)+I(\mathcal O;Y^Q\mid U,W)\\
  &=I(\mathcal O;Y^Q\mid W)
    +I(U;Y^Q\mid\mathcal O,W).
\end{align}
Response sufficiency makes
$I(\mathcal O;Y^Q\mid U,W)=0$, which proves
\cref{eq:value-access-mi-decomposition}. Nonnegativity of conditional mutual
information gives the risk ladder. Response sufficiency also makes
$\mathcal O\to U\to Y^Q$ a conditional Markov chain given $W$.
Conditional data processing gives
$I(\mathcal O;Y^Q\mid W)\leq I(U;\mathcal O\mid W)$, while
\cref{eq:value-access-mi-decomposition} gives the other upper bound in
\cref{eq:value-access-data-processing}.
\end{proof}

The mutual-information decomposition itself requires response sufficiency but
not the external-query condition. The latter is needed when the exact response
conditional is represented using the fixed belief
$P(U\in du\mid\mathcal O)$ rather than
$P(U\in du\mid\mathcal O,W)$.

\subsection{Proof of the End-to-End Approximation Result}
\label{app:end-to-end-approximation-proof}

Assume the unit, evidence, query, and outcome spaces are standard Borel and use
compatible regular-conditional versions. The finite-risk statement below uses
conditional densities with respect to a common sigma-finite reference measure;
the component bound may instead be read in the extended relative-entropy sense.

\begin{proof}[Proof of \cref{prop:end-to-end-approximation}]
Conditioning on $(\mathcal O,W)$, the difference between the learner's
expected log loss and the Bayes log loss is
\begin{equation}
  E\!\left[
    \log\frac{p(Y^Q\mid\mathcal O,W)}
                  {\widehat p_{\theta,\phi}(Y^Q\mid\mathcal O;W)}
    \,\middle|\,\mathcal O,W\right],
\end{equation}
which is
$D_{\rm KL}(M_{\mathcal O,W}\|\widehat M_{\mathcal O,W})$ under
\cref{eq:value-access-response-sufficiency,eq:external-query-contract}.
Averaging proves \cref{eq:learner-predictive-kl}. Adding and subtracting
$R_{\mathcal O,W}^\star$ and applying
\cref{prop:value-access-decomposition} proves
\cref{eq:learner-achieved-value,eq:learner-oracle-decomposition}.

For the component bound, at fixed $(o,w)$ define latent-joint laws
\begin{equation}
  J_{o,w}(du,dy):=P_o(du)K_{u,w}(dy),
  \qquad
  \widehat J_{o,w}(du,dy):=Q_o(du)\widehat K_{u,w}(dy).
  \label{eq:latent-joint-laws}
\end{equation}
The relative-entropy chain rule gives
\begin{align}
  D_{\rm KL}(J_{o,w}\|\widehat J_{o,w})
  &=D_{\rm KL}(P_o\|Q_o)
    +\int D_{\rm KL}(K_{u,w}\|\widehat K_{u,w})P_o(du).
  \label{eq:latent-joint-kl-chain}
\end{align}
The measurable projection $(u,y)\mapsto y$ sends $J_{o,w}$ to $M_{o,w}$
and $\widehat J_{o,w}$ to $\widehat M_{o,w}$. Data processing therefore gives
\begin{equation}
  D_{\rm KL}(M_{o,w}\|\widehat M_{o,w})
  \leq D_{\rm KL}(J_{o,w}\|\widehat J_{o,w}).
  \label{eq:latent-joint-kl-marginal}
\end{equation}
Averaging \cref{eq:latent-joint-kl-chain,eq:latent-joint-kl-marginal} over the
deployment law proves \cref{eq:latent-joint-kl-deployment}.
\end{proof}

This upper bound is a marginal contraction of a latent-joint discrepancy. It
need not be tight: distinct latent decompositions can have identical response
marginals, as the next subsection makes explicit.

\subsection{Single-Row Marginal Collapse and Linked-Pair Separation}
\label{app:single-row-linked-proof}

For compatible versions of the regular conditional kernels, the law of total
probability after conditioning on $X=x$ proves
\cref{eq:exact-rowwise-collapse}, with its $P_X$-almost-everywhere
qualification. For any sample-only conditional law $R(dy\mid x)$, choosing
$K_u^R(dy\mid x):=R(dy\mid x)$ for every $u$ in an unrestricted response-kernel
class returns $R$ after integration against every learner belief. Thus marginal
response fit alone does not select a formed token, a which-unit belief, or a
fixed-unit response decomposition.

\begin{proof}[Proof of \cref{prop:single-row-collapse}]
In world $\mathsf H$,
\begin{equation}
  P_{\mathsf H}(Y=1)
  =\tfrac12(\tfrac14+\tfrac34)=\tfrac12,
\end{equation}
which is also the response probability in world $\mathsf P$. Independence of
the fresh unit draws and their event realizations gives
\cref{eq:single-row-equal-laws} for every $n$.

Let a possibly randomized test $\varphi_n$ declare $\mathsf H$ with
probability $\varphi_n(Y_1,\ldots,Y_n)$ and define
\begin{equation}
  \alpha_n=E_{\mathsf P}[\varphi_n],
  \qquad
  \beta_n=E_{\mathsf H}[1-\varphi_n].
\end{equation}
Equality of the observable laws implies
$E_{\mathsf H}\varphi_n=E_{\mathsf P}\varphi_n$, so
$\alpha_n+\beta_n=1$ and
$\max\{\alpha_n,\beta_n\}\geq1/2$. Equality is attained by randomizing equally
between the two worlds.

For a linked pair, conditional independence gives
\begin{equation}
  \operatorname{Cov}(Y_1,Y_2)=\operatorname{Var}(p_U).
  \label{eq:linked-pair-covariance-general}
\end{equation}
This variance is $1/16$ in world $\mathsf H$ and zero in world $\mathsf P$.
More explicitly, their pair probabilities are
\begin{equation}
\begin{array}{c|cccc}
 & 00 & 01 & 10 & 11\\
\hline
\mathsf H & 5/16 & 3/16 & 3/16 & 5/16\\
\mathsf P & 1/4  & 1/4  & 1/4  & 1/4.
\end{array}
\label{eq:linked-pair-probability-table}
\end{equation}
If $M$ linked pairs are sampled independently from independently drawn units,
define
\begin{equation}
  \widehat a_M
  :=\frac1M\sum_{j=1}^M\mathbf 1\{Y_{j1}=Y_{j2}\}.
  \label{eq:linked-pair-agreement-rate}
\end{equation}
The test that declares $\mathsf H$ when $\widehat a_M>9/16$ is consistent;
Hoeffding's inequality bounds each error by $\exp(-M/128)$. Finally,
$P(Y_1=1)=E[p_U]$ and $P(Y_1=1,Y_2=1)=E[p_U^2]$, proving
\cref{eq:linked-pair-variance-identification}.
\end{proof}

The full mixing law is not identified by these two moments without further
restrictions. For example,
\begin{equation}
  \tfrac12\delta_{1/4}+\tfrac12\delta_{3/4}
  \quad\text{and}\quad
  \tfrac18\delta_0+\tfrac34\delta_{1/2}+\tfrac18\delta_1
  \label{eq:same-two-moments-mixing-laws}
\end{equation}
are distinct laws for $p_U$ with the same first moment $1/2$ and second moment
$5/16$.

Trusted same-unit linkage and the conditional product law are essential to this
positive statement. Arbitrary within-unit dependence, linkage error,
response-relevant temporal state, or informative observation can invalidate
the displayed separation. The result identifies neither the full mixing law nor
membership of a realized unit in either response class. It therefore does not
contradict the need for structured mixture assumptions in classical positive
identifiability results \citep{teicher1963identifiability}.

\subsection{Attribution information and response dependence are independent}
\label{app:two-attribution-response-axes}

For chosen regular-conditional versions, evidence is uninformative about the
unit exactly when
\begin{equation}
  U\perp\!\!\!\perp\mathcal O
  \quad\Longleftrightarrow\quad
  P(U\in\cdot\mid\mathcal O)=\Pi(\cdot)
  \quad\text{almost surely}.
  \label{eq:appendix-evidence-collapse}
\end{equation}
This statement concerns the conditional law under the data-generating
distribution. It constrains a formed token only insofar as that token is
required to carry identity information, and it constrains a learner belief
$Q_\phi$ only when that belief is calibrated to the declared identity target.

Independently, suppose the response family is unit-homogeneous at a declared
query $w$: $K_u^w=K^w$ for $\Pi$-almost every $u$. Then every learner belief
with $Q_\phi(\cdot\mid\mathcal O)\ll\Pi$ satisfies
\begin{equation}
  \int_{\mathcal U}K_u^w\,
    Q_\phi(du\mid\mathcal O)=K^w.
  \label{eq:appendix-response-collapse}
\end{equation}
Neither condition implies the other. Informative attribution can coexist with
a unit-homogeneous response, and unit-heterogeneous responses can coexist with
uninformative evidence. Moreover, cancellation in one particular mixture does
not by itself establish response homogeneity. These are two collapse
conditions for two different learner formulations, distinct again from the single-row
marginal collapse in \cref{prop:single-row-collapse}.

\subsection{Proof and qualifications for mixture stability}
\label{app:mixture-stability-proof}

Fix $o$, $w$, and $\mathcal U_o$ as in \cref{prop:mixture-stability}, and
suppress $w$ by writing $K_u:=K_{u,w}$ and
$\widehat K_u:=\widehat K_{u,w}$. Assume the unit and outcome
spaces are standard Borel, $K$ and $\widehat K$ are measurable Markov
kernels on the chosen versions, and the displayed total-variation integrand is
measurable. These conditions ensure that all mixtures and integrals below are
well-defined. The common full-mass restriction ensures that replacing
$\mathcal U$ by $\mathcal U_o$ in the diameter does not discard mixture mass.

\begin{proof}[Proof of \cref{prop:mixture-stability}]
The triangle inequality gives
\begin{align}
d_{\rm TV}(Q_o\widehat K,P_oK)
&\leq d_{\rm TV}(Q_o\widehat K,Q_oK)
     +d_{\rm TV}(Q_oK,P_oK),
\label{eq:mixture-stability-triangle}
\end{align}
where, for example,
$Q_oK=\int_{\mathcal U}K_u Q_o(du)$. For the first term, the definition of
total variation and the triangle inequality for integrals yield
\begin{align}
d_{\rm TV}(Q_o\widehat K,Q_oK)
&\leq
\int_{\mathcal U}
d_{\rm TV}(\widehat K_u,K_u)Q_o(du).
\label{eq:mixture-stability-kernel-part}
\end{align}

For the second term, let $\Delta_o=Q_o-P_o$ and
$t=d_{\rm TV}(Q_o,P_o)$. If $t=0$ there is nothing to prove. Otherwise the
Jordan decomposition of the zero-mass signed measure $\Delta_o$ can be written as
$\Delta_o=t(\alpha-\beta)$ for probability measures $\alpha$ and $\beta$ supported
on $\mathcal U_o$. For any measurable outcome event $A$,
\begin{align}
\left|(\alpha K)(A)-(\beta K)(A)\right|
&=\left|\int\!\!\int
  \bigl[K_u(A)-K_v(A)\bigr]\alpha(du)\beta(dv)\right| \\
&\leq \sup_{u,v\in\mathcal U_o}d_{\rm TV}(K_u,K_v).
\end{align}
Taking the supremum over $A$ and multiplying by $t$ gives
\begin{equation}
d_{\rm TV}(Q_oK,P_oK)
\leq
\left[\sup_{u,v\in\mathcal U_o}d_{\rm TV}(K_u,K_v)\right]
d_{\rm TV}(Q_o,P_o).
\label{eq:mixture-stability-belief-part}
\end{equation}
Combining \cref{eq:mixture-stability-triangle,eq:mixture-stability-kernel-part,eq:mixture-stability-belief-part}
proves
\cref{eq:mixture-stability}.
\end{proof}

The diameter term is response-input-specific. It is zero exactly when the selected
versions of the true response law agree across the units included in the
supremum. If $Q_o$ assigns mass outside the region on which the true kernel is
scientifically specified, that is a support failure rather than a small
attribution error; the proposition does not repair it. For unbounded losses,
including unrestricted log loss, total-variation stability alone does not
provide a finite excess-risk bound without additional boundedness or density
conditions.

\subsection{Oracle, deployed, row-weighted, and unit-weighted risks}

If the true $u$ is supplied to the response predictor, the population-first
oracle risk separates which-individual and fixed-individual variation:
\begin{equation}
  R_{\Pi}^{\mathrm{oracle}}(f)
  =\int_{\mathcal U}
    \left[\int \ell\!\left(y,f(x;u)\right)K_u(dx,dy)\right]
    \Pi(du).
  \label{eq:appendix-population-risk}
\end{equation}
The outer integral varies $U$; the inner integral retains exogenous/event
variation at fixed $u$.  This is a diagnostic for the response layer, not the
deployed risk of a learner that knows only a formed token, or only $Q_\phi$
in the identity-belief specialization.

Let a declared deployment experiment generate factual evidence
$\mathcal O$, a response query $X^Q$, any explicit response context $C^Q$,
and the target response $Y^Q$. Under log loss, the default deployed risk of
the pair $(T_\phi,R_\theta)$ is
\begin{equation}
  R_{\mathrm{dep}}(\theta,\phi)
  =
  E_{P_{\mathrm{dep}}}\!\left[
    -\log\!\left\{
      \int_{\mathcal Z}
      r_\theta(Y^Q\mid X^Q,C^Q,z)\,
      T_\phi(dz\mid\mathcal O,C^Q)
    \right\}
  \right],
  \label{eq:appendix-deployed-risk-token}
\end{equation}
when a density $r_\theta$ of $R_\theta$ exists. In the identity-belief
specialization used by the theorems, the same experiment scores
\begin{equation}
  R_{\mathrm{dep}}(\theta,\phi)
  =
  E_{P_{\mathrm{dep}}}\!\left[
    -\log\!\left\{
      \int_{\mathcal U}
      p_\theta(Y^Q\mid X^Q,C^Q,u)
      Q_\phi(du\mid\mathcal O)
    \right\}
  \right].
  \label{eq:appendix-deployed-risk}
\end{equation}
Other proper scores or decision losses may replace log loss. The deployment
experiment and evidence cutoff are part of the risk definition.
Subpopulation-law quality, oracle fixed-$u$ response quality, and deployed
marginalized quality are different quantities; the last alone need not identify
the first two. A random-row design may also induce a size-biased unit law, so its
empirical risk need not estimate either declared population target.

\subsection{Repeated-unit weighting and splitting}
\label{app:repeated-unit-evaluation-proofs}

In an observed-attribution dataset, let
$M=|\mathcal U_{\mathcal D}|$. For each observed unit define
\begin{equation}
  I(u)=\{i:u_i=u\},
  \qquad
  m(u)=|I(u)|.
  \label{eq:appendix-unit-count}
\end{equation}
For sample losses $\ell_i$ define the mean loss of unit $u$ by
$\bar\ell(u)=m(u)^{-1}\sum_{i\in I(u)}\ell_i$. The row-average and
unit-average risks are
\begin{equation}
  \widehat R_{\mathrm{row}}
  =\frac{1}{N}\sum_{i=1}^{N}\ell_i
  =\sum_{u\in\mathcal U_{\mathcal D}}\frac{m(u)}{N}\bar\ell(u),
  \qquad
  \widehat R_{\mathrm{unit}}
  =\frac{1}{M}\sum_{u\in\mathcal U_{\mathcal D}}\bar\ell(u).
  \label{eq:appendix-risks}
\end{equation}
\begin{proof}[Proof of \cref{prop:risk-weighting}]
The two expressions are linear in $(\bar\ell(u))_{u\in\mathcal U_{\mathcal D}}$,
so they agree for every such vector exactly when their coefficients agree for
every $u$, that is, when $m(u)/N=1/M$.
\end{proof}

Neither objective is universally
correct; the deployment estimand determines which weighting is appropriate.

\begin{proof}[Proof of \cref{prop:split}]
The two excluded events are that all $m(u)$ records enter training and that all
enter test; they are disjoint and have probabilities $p^{m(u)}$ and
$(1-p)^{m(u)}$.
\end{proof}

A record-wise split can therefore mix new observations of known units with any
new-unit cases. An unseen-unit generalization estimand uses a test construction
that is explicitly unit-disjoint from training, whether by a whole-unit split,
an external new-unit cohort, or another declared protocol.

\section{Observability Conditions}
\label{app:implementation-boundaries}

\paragraph{Component targets and training regimes.}
The world target $P_u^\star$ and the learner pair $(T_\phi,R_\theta)$ remain
distinct. Separate token supervision, oracle-attribution response training, and
marginalized end-to-end training are different identification regimes. A
which-unit belief has its own target only when the protocol asks it to
approximate $P(U\in du\mid\mathcal O)$. Under trusted direct access there is
no identity model to estimate: learning the lookup row
$Z_{u(k)}^c=z_\theta(k)$ fits a token inside $R_\theta$, not $Q_\phi$. Any
history-derived state used at answer time must appear in the tokenizer context,
in $c^Q$, or inside the shared form; otherwise response sufficiency is not
justified. In \cref{sec:consequences} the complete response input is written
$W$, not $Z$, so that it is not confused with the unit token $Z_u^c$.

\paragraph{Support and observational equivalence.}
The population law may support an individual value $u^\star$ even when no
training record is attributed to it. Support alone does not concentrate
$P(U\in du\mid\mathcal O)$ at $u^\star$; the available event information may
instead imply a diffuse token, a population fallback, or abstention. Likewise,
two distinct individual values may induce the same response law for every query
and context admitted by a study. They are observationally indistinguishable
under that study while remaining distinct units.

\paragraph{Fixed evidence and shared-unit bundles.}
The task declaration specifies the unit boundary, the events linked to one fixed
individual, the answer-time evidence cutoff, and the query family. Within a
declared same-unit query family $\mathfrak Q$, the same $u$ and the same factual
evidence $\mathcal O$ are retained while $x^Q\in\mathfrak Q$ varies. A new
factual event may update the formed token, whereas comparing alternative
response queries uses the token formed at the fixed evidence cutoff. Shared
attribution of several rows to one unit does not by itself imply that those
rows are conditionally independent given $U$.

\paragraph{The fixed-individual kernel retains exogenous variation.}
The kernel $P_u^\star(dy\mid x^Q,c^Q)$ may remain stochastic after
$U=u$ is fixed. Residual within-unit variation is not which-unit variation.

\section{Evaluation Checklist}
\label{app:evaluation-checklist}

A reproducible evaluation should pre-specify:
\begin{enumerate}[leftmargin=*,itemsep=1pt,topsep=3pt]
  \item the unit population and the span over which unit identity persists;
  \item the factual-evidence cutoff, response query and context, and current
  target;
  \item the access regime---direct access or unit abduction---and what
  attribution truth is available for evaluation;
  \item the evaluated target: fixed-unit response, deployed marginalized
  prediction, or a unit-level estimand;
  \item whether the split targets known-unit/new-event or new-unit
  generalization;
  \item a matched unit-omitting baseline and any negative control required by the
  claim.
\end{enumerate}

A learnable unit-conditioned specification additionally names:
\begin{enumerate}[leftmargin=*,itemsep=1pt,topsep=3pt]
  \item what is shared across units, here the form $R_\theta$, a simple
        relation in the token, with a linear predictor as the running instance
        (\cref{def:linear-shared-form});
  \item which properties of the query--response relation vary with the unit,
        and are therefore expressed by $Z_u^c$;
  \item which properties remain invariant across units;
  \item how finitely many same-unit observations constrain that unit's token
        and hence its response law;
  \item how the shared form and token space support prediction for a unit not
        previously observed; and
  \item how evaluation distinguishes a learned unit-dependent relationship
        from memorization of an identifier.
\end{enumerate}

A product likelihood adds a conditional-factorization assumption. The response
law is observational by default; an interventional reading requires assignment
and identification assumptions. Pre-answer evidence excludes current targets
and unavailable post-query measurements.

%% file: references.bib
@article{hurlbert1984pseudoreplication,
  author  = {Hurlbert, Stuart H.},
  title   = {Pseudoreplication and the Design of Ecological Field Experiments},
  journal = {Ecological Monographs},
  volume  = {54},
  number  = {2},
  pages   = {187--211},
  year    = {1984},
  doi     = {10.2307/1942661}
}

@article{laird1982random,
  author  = {Laird, Nan M. and Ware, James H.},
  title   = {Random-Effects Models for Longitudinal Data},
  journal = {Biometrics},
  volume  = {38},
  number  = {4},
  pages   = {963--974},
  year    = {1982},
  doi     = {10.2307/2529876}
}

@article{caruana1997multitask,
  author  = {Caruana, Rich},
  title   = {Multitask Learning},
  journal = {Machine Learning},
  volume  = {28},
  pages   = {41--75},
  year    = {1997},
  doi     = {10.1023/A:1007379606734}
}

@article{baxter2000model,
  author  = {Baxter, Jonathan},
  title   = {A Model of Inductive Bias Learning},
  journal = {Journal of Artificial Intelligence Research},
  volume  = {12},
  pages   = {149--198},
  year    = {2000},
  doi     = {10.1613/jair.731}
}

@book{bishop2006pattern,
  author    = {Bishop, Christopher M.},
  title     = {Pattern Recognition and Machine Learning},
  series    = {Information Science and Statistics},
  publisher = {Springer},
  address   = {New York, NY},
  year      = {2006},
  isbn      = {978-0-387-31073-2},
  url       = {https://link.springer.com/book/9780387310732}
}

@book{hastie2009elements,
  author    = {Hastie, Trevor and Tibshirani, Robert and Friedman, Jerome},
  title     = {The Elements of Statistical Learning: Data Mining, Inference, and Prediction},
  edition   = {2},
  series    = {Springer Series in Statistics},
  publisher = {Springer},
  address   = {New York, NY},
  year      = {2009},
  doi       = {10.1007/978-0-387-84858-7},
  isbn      = {978-0-387-84857-0}
}

@book{shalevshwartz2014understanding,
  author    = {Shalev-Shwartz, Shai and Ben-David, Shai},
  title     = {Understanding Machine Learning: From Theory to Algorithms},
  publisher = {Cambridge University Press},
  address   = {Cambridge},
  year      = {2014},
  doi       = {10.1017/CBO9781107298019},
  isbn      = {978-1-107-05713-5}
}

@book{murphy2012machine,
  author    = {Murphy, Kevin P.},
  title     = {Machine Learning: A Probabilistic Perspective},
  series    = {Adaptive Computation and Machine Learning},
  publisher = {MIT Press},
  address   = {Cambridge, MA},
  year      = {2012},
  isbn      = {978-0-262-01802-9},
  url       = {https://www.cs.ubc.ca/~murphyk/MLbook/}
}

@article{bengio2013representation,
  author  = {Bengio, Yoshua and Courville, Aaron and Vincent, Pascal},
  title   = {Representation Learning: A Review and New Perspectives},
  journal = {IEEE Transactions on Pattern Analysis and Machine Intelligence},
  volume  = {35},
  number  = {8},
  pages   = {1798--1828},
  year    = {2013},
  doi     = {10.1109/TPAMI.2013.50}
}

@article{dempster1977maximum,
  author  = {Dempster, Arthur P. and Laird, Nan M. and Rubin, Donald B.},
  title   = {Maximum Likelihood from Incomplete Data via the {EM} Algorithm},
  journal = {Journal of the Royal Statistical Society. Series B (Methodological)},
  volume  = {39},
  number  = {1},
  pages   = {1--22},
  year    = {1977},
  doi     = {10.1111/j.2517-6161.1977.tb01600.x}
}

@article{teicher1963identifiability,
  author  = {Teicher, Henry},
  title   = {Identifiability of Finite Mixtures},
  journal = {The Annals of Mathematical Statistics},
  volume  = {34},
  number  = {4},
  pages   = {1265--1269},
  year    = {1963},
  doi     = {10.1214/aoms/1177703862}
}

@article{shimodaira2000covariate,
  author  = {Shimodaira, Hidetoshi},
  title   = {Improving Predictive Inference under Covariate Shift by Weighting the Log-Likelihood Function},
  journal = {Journal of Statistical Planning and Inference},
  volume  = {90},
  number  = {2},
  pages   = {227--244},
  year    = {2000},
  doi     = {10.1016/S0378-3758(00)00115-4}
}

@article{bendavid2010theory,
  author  = {Ben-David, Shai and Blitzer, John and Crammer, Koby and Kulesza, Alex and Pereira, Fernando and Wortman Vaughan, Jennifer},
  title   = {A Theory of Learning from Different Domains},
  journal = {Machine Learning},
  volume  = {79},
  pages   = {151--175},
  year    = {2010},
  doi     = {10.1007/s10994-009-5152-4}
}

@article{peters2016invariant,
  author  = {Peters, Jonas and B{\"u}hlmann, Peter and Meinshausen, Nicolai},
  title   = {Causal Inference Using Invariant Prediction: Identification and Confidence Intervals},
  journal = {Journal of the Royal Statistical Society: Series B (Statistical Methodology)},
  volume  = {78},
  number  = {5},
  pages   = {947--1012},
  year    = {2016},
  doi     = {10.1111/rssb.12167}
}

@article{gneiting2007strictly,
  author  = {Gneiting, Tilmann and Raftery, Adrian E.},
  title   = {Strictly Proper Scoring Rules, Prediction, and Estimation},
  journal = {Journal of the American Statistical Association},
  volume  = {102},
  number  = {477},
  pages   = {359--378},
  year    = {2007},
  doi     = {10.1198/016214506000001437}
}

@book{gelman2006data,
  author    = {Gelman, Andrew and Hill, Jennifer},
  title     = {Data Analysis Using Regression and Multilevel/Hierarchical Models},
  publisher = {Cambridge University Press},
  year      = {2006},
  doi       = {10.1017/CBO9780511790942}
}

@article{saeb2017need,
  author  = {Saeb, Sohrab and Lonini, Luca and Jayaraman, Arun and Mohr, David C. and Kording, Konrad P.},
  title   = {The Need to Approximate the Use-Case in Clinical Machine Learning},
  journal = {GigaScience},
  volume  = {6},
  number  = {5},
  pages   = {1--9},
  year    = {2017},
  doi     = {10.1093/gigascience/gix019}
}

@article{roberts2017crossvalidation,
  author  = {Roberts, David R. and Bahn, Volker and Ciuti, Simone and Boyce, Mark S. and Elith, Jane and Guillera-Arroita, Gurutzeta and Hauenstein, Severin and Lahoz-Monfort, Jos\'e J. and Schr\"oder, Boris and Thuiller, Wilfried and Warton, David I. and Wintle, Brendan A. and Hartig, Florian and Dormann, Carsten F.},
  title   = {Cross-Validation Strategies for Data with Temporal, Spatial, Hierarchical, or Phylogenetic Structure},
  journal = {Ecography},
  volume  = {40},
  number  = {8},
  pages   = {913--929},
  year    = {2017},
  doi     = {10.1111/ecog.02881}
}

@article{jacobs1991adaptive,
  author  = {Jacobs, Robert A. and Jordan, Michael I. and Nowlan, Steven J. and Hinton, Geoffrey E.},
  title   = {Adaptive Mixtures of Local Experts},
  journal = {Neural Computation},
  volume  = {3},
  number  = {1},
  pages   = {79--87},
  year    = {1991},
  doi     = {10.1162/neco.1991.3.1.79},
  url     = {https://www.cs.toronto.edu/~hinton/absps/jjnh91.pdf}
}

@inproceedings{kingma2014autoencoding,
  author    = {Kingma, Diederik P. and Welling, Max},
  title     = {Auto-Encoding Variational {Bayes}},
  booktitle = {International Conference on Learning Representations},
  year      = {2014},
  url       = {https://arxiv.org/abs/1312.6114}
}

@inproceedings{edwards2017neural,
  author    = {Edwards, Harrison and Storkey, Amos},
  title     = {Towards a Neural Statistician},
  booktitle = {International Conference on Learning Representations},
  year      = {2017},
  url       = {https://arxiv.org/abs/1606.02185}
}

@inproceedings{bouchacourt2018mlvae,
  author    = {Bouchacourt, Diane and Tomioka, Ryota and Nowozin, Sebastian},
  title     = {{Multi-Level Variational Autoencoder}: Learning Disentangled Representations from Grouped Observations},
  booktitle = {Proceedings of the Thirty-Second AAAI Conference on Artificial Intelligence},
  pages     = {2095--2102},
  publisher = {AAAI Press},
  year      = {2018},
  doi       = {10.1609/aaai.v32i1.11867}
}

@article{garnelo2018neural,
  author  = {Garnelo, Marta and Schwarz, Jonathan and Rosenbaum, Dan and Viola, Fabio and Rezende, Danilo J. and Eslami, S. M. Ali and Teh, Yee Whye},
  title   = {Neural Processes},
  journal = {arXiv preprint arXiv:1807.01622},
  year    = {2018},
  url     = {https://arxiv.org/abs/1807.01622}
}

@inproceedings{garnelo2018conditional,
  author    = {Garnelo, Marta and Rosenbaum, Dan and Maddison, Chris J. and Ramalho, Tiago and Saxton, David and Shanahan, Murray and Teh, Yee Whye and Rezende, Danilo J. and Eslami, S. M. Ali},
  title     = {Conditional Neural Processes},
  booktitle = {Proceedings of the 35th International Conference on Machine Learning},
  series    = {Proceedings of Machine Learning Research},
  volume    = {80},
  pages     = {1704--1713},
  publisher = {PMLR},
  year      = {2018},
  url       = {https://proceedings.mlr.press/v80/garnelo18a.html}
}

@inproceedings{gordon2019versa,
  author    = {Gordon, Jonathan and Bronskill, John and Bauer, Matthias and Nowozin, Sebastian and Turner, Richard E.},
  title     = {Meta-Learning Probabilistic Inference for Prediction},
  booktitle = {International Conference on Learning Representations},
  year      = {2019},
  url       = {https://arxiv.org/abs/1805.09921}
}

@inproceedings{mnih2007pmf,
  author    = {Mnih, Andriy and Salakhutdinov, Ruslan R.},
  title     = {Probabilistic Matrix Factorization},
  booktitle = {Advances in Neural Information Processing Systems 20},
  pages     = {1257--1264},
  year      = {2007},
  url       = {https://proceedings.neurips.cc/paper/2007/hash/d7322ed717dedf1eb4e6e52a37ea7bcd-Abstract.html}
}

@inproceedings{salakhutdinov2008bpmf,
  author    = {Salakhutdinov, Ruslan and Mnih, Andriy},
  title     = {{Bayesian} Probabilistic Matrix Factorization Using {Markov Chain Monte Carlo}},
  booktitle = {Proceedings of the 25th International Conference on Machine Learning},
  pages     = {880--887},
  publisher = {ACM},
  year      = {2008},
  doi       = {10.1145/1390156.1390267}
}

@article{koren2009matrix,
  author  = {Koren, Yehuda and Bell, Robert and Volinsky, Chris},
  title   = {Matrix Factorization Techniques for Recommender Systems},
  journal = {Computer},
  volume  = {42},
  number  = {8},
  pages   = {30--37},
  year    = {2009},
  doi     = {10.1109/MC.2009.263}
}

@inproceedings{kang2018sasrec,
  author    = {Kang, Wang-Cheng and McAuley, Julian},
  title     = {Self-Attentive Sequential Recommendation},
  booktitle = {2018 IEEE International Conference on Data Mining},
  pages     = {197--206},
  publisher = {IEEE},
  year      = {2018},
  doi       = {10.1109/ICDM.2018.00035}
}

@inproceedings{liang2018multvae,
  author    = {Liang, Dawen and Krishnan, Rahul G. and Hoffman, Matthew D. and Jebara, Tony},
  title     = {Variational Autoencoders for Collaborative Filtering},
  booktitle = {Proceedings of the 2018 World Wide Web Conference},
  pages     = {689--698},
  publisher = {ACM},
  year      = {2018},
  doi       = {10.1145/3178876.3186150}
}

@inproceedings{he2017neural,
  author    = {He, Xiangnan and Liao, Lizi and Zhang, Hanwang and Nie, Liqiang and Hu, Xia and Chua, Tat-Seng},
  title     = {Neural Collaborative Filtering},
  booktitle = {Proceedings of the 26th International Conference on World Wide Web},
  pages     = {173--182},
  publisher = {International World Wide Web Conferences Steering Committee},
  year      = {2017},
  doi       = {10.1145/3038912.3052569}
}

@inproceedings{oh2019modeling,
  author    = {Oh, Seong Joon and Murphy, Kevin P. and Pan, Jiyan and Roth, Joseph and Schroff, Florian and Gallagher, Andrew C.},
  title     = {Modeling Uncertainty with Hedged Instance Embeddings},
  booktitle = {7th International Conference on Learning Representations},
  publisher = {OpenReview.net},
  year      = {2019},
  url       = {https://openreview.net/forum?id=r1xQQhAqKX}
}

@inproceedings{shi2019probabilistic,
  author    = {Shi, Yichun and Jain, Anil K.},
  title     = {Probabilistic Face Embeddings},
  booktitle = {2019 IEEE/CVF International Conference on Computer Vision},
  pages     = {6901--6910},
  publisher = {IEEE},
  year      = {2019},
  doi       = {10.1109/ICCV.2019.00700}
}

@article{steorts2016entity,
  author  = {Steorts, Rebecca C. and Hall, Rob and Fienberg, Stephen E.},
  title   = {A {Bayesian} Approach to Graphical Record Linkage and De-duplication},
  journal = {Journal of the American Statistical Association},
  volume  = {111},
  number  = {516},
  pages   = {1660--1672},
  year    = {2016},
  doi     = {10.1080/01621459.2015.1105807}
}

@incollection{steorts2018generalized,
  author    = {Steorts, Rebecca C. and Tancredi, Andrea and Liseo, Brunero},
  title     = {Generalized {Bayesian} Record Linkage and Regression with Exact Error Propagation},
  booktitle = {Privacy in Statistical Databases},
  editor    = {Domingo-Ferrer, Josep and Montes, Francisco},
  series    = {Lecture Notes in Computer Science},
  pages     = {297--313},
  publisher = {Springer International Publishing},
  year      = {2018},
  doi       = {10.1007/978-3-319-99771-1_20}
}

@article{kaplan2022practical,
  author  = {Kaplan, Andee and Betancourt, Brenda and Steorts, Rebecca C.},
  title   = {A Practical Approach to Proper Inference with Linked Data},
  journal = {The American Statistician},
  volume  = {76},
  number  = {4},
  pages   = {384--393},
  year    = {2022},
  doi     = {10.1080/00031305.2022.2041482}
}

@article{li2020ditto,
  author  = {Li, Yuliang and Li, Jinfeng and Suhara, Yoshihiko and Doan, AnHai and Tan, Wang-Chiew},
  title   = {{Ditto}: Deep Entity Matching with Pre-Trained Language Models},
  journal = {Proceedings of the VLDB Endowment},
  volume  = {14},
  number  = {1},
  pages   = {50--60},
  year    = {2020},
  doi     = {10.14778/3421424.3421431}
}

@article{rubin1974causal,
  author  = {Rubin, Donald B.},
  title   = {Estimating Causal Effects of Treatments in Randomized and Nonrandomized Studies},
  journal = {Journal of Educational Psychology},
  volume  = {66},
  number  = {5},
  pages   = {688--701},
  year    = {1974},
  doi     = {10.1037/h0037350}
}

@article{gong2024discoscm,
  author        = {Gong, Heyang and Lu, Chaochao and Zhang, Yu},
  title         = {Distribution-Consistency Structural Causal Models},
  journal       = {arXiv preprint arXiv:2401.15911},
  year          = {2024},
  eprint        = {2401.15911},
  archiveprefix = {arXiv},
  url           = {https://arxiv.org/abs/2401.15911}
}

@book{pearl2009causality,
  author    = {Pearl, Judea},
  title     = {Causality: Models, Reasoning, and Inference},
  edition   = {2},
  publisher = {Cambridge University Press},
  year      = {2009},
  doi       = {10.1017/CBO9780511803161}
}

@inproceedings{pawlowski2020dscm,
  author    = {Pawlowski, Nick and Coelho de Castro, Daniel and Glocker, Ben},
  title     = {{Deep Structural Causal Models} for Tractable Counterfactual Inference},
  booktitle = {Advances in Neural Information Processing Systems 33},
  pages     = {857--869},
  year      = {2020},
  url       = {https://proceedings.neurips.cc/paper/2020/hash/0987b8b338d6c90bbedd8631bc499221-Abstract.html}
}

@inproceedings{dai2019bridging,
  author    = {Dai, Wang-Zhou and Xu, Qiu-Ling and Yu, Yang and Zhou, Zhi-Hua},
  title     = {Bridging Machine Learning and Logical Reasoning by Abductive Learning},
  booktitle = {Advances in Neural Information Processing Systems 32},
  pages     = {2811--2822},
  year      = {2019},
  url       = {https://proceedings.neurips.cc/paper/2019/hash/9c19a2aa1d84e04b0bd4bc888792bd1e-Abstract.html}
}

@book{cover2005elements,
  author    = {Cover, Thomas M. and Thomas, Joy A.},
  title     = {Elements of Information Theory},
  edition   = {2},
  publisher = {Wiley},
  year      = {2005},
  month     = apr,
  isbn      = {9780471748823},
  doi       = {10.1002/047174882X},
  url       = {https://doi.org/10.1002/047174882X}
}
